\documentclass[lettersize,journal]{IEEEtran}  
\usepackage{microtype}
\usepackage{graphicx}
\usepackage{booktabs}
\usepackage{multirow}
\usepackage{algorithm}
\usepackage{algorithmic}
\usepackage[caption=false,font=normalsize,labelfont=sf,textfont=sf]{subfig}
\usepackage{float}
\usepackage{placeins}
\usepackage{bm}
\usepackage{amsmath}
\usepackage{amssymb}
\usepackage{mathtools}
\usepackage{amsthm}
\usepackage{cite}
\usepackage{stfloats}

\theoremstyle{plain}
\newtheorem{theorem}{Theorem}[section]

\theoremstyle{definition}
\newtheorem{definition}[theorem]{Definition}

\theoremstyle{remark}

\begin{document}

\title{
Sim2O: Efficient Offline-to-Online MARL via Joint Action Composition
}
\author{}

\author{Bingchang Song$^1$, Yiqin Yang$^2$\\
$^1$ Department of Computer Science and Engineering, Washington University in St. Louis\\
$^2$ The Key Laboratory of Cognition and Decision Intelligence for Complex Systems, Institute of Automation, Chinese Academy of Sciences, Beijing, China
\thanks{Corresponding author: Yiqin Yang.}%
}

\maketitle

\begin{abstract}
Offline-to-online adaptation serves as a pivotal paradigm for mitigating the prohibitive cost of online exploration by bootstrapping reinforcement learning from offline datasets. While this paradigm has been extensively studied in single-agent settings, its extension to Multi-Agent Reinforcement Learning (MARL) remains largely unexplored, despite its critical relevance to complex coordinated decision-making. To bridge this gap, we introduce Sim2O, an elegant and minimalist framework for offline-to-online MARL. Rather than treating adaptation as a monolithic joint decision, Sim2O conceptualizes it as a compositional process. Specifically, candidate joint actions are synthesized by dynamically blending offline and online action proposals across agents. By leveraging a centralized value function to evaluate these hybrid combinations, Sim2O identifies high-value coordination strategies without requiring auxiliary training objectives or structural overhead. Empirical evaluations across diverse benchmarks demonstrate that Sim2O significantly outperforms existing baselines, underscoring that a minimalist design is not only viable but highly effective for multi-agent offline-to-online adaptation.
\end{abstract}

\begin{figure*}[t]
    \centering
    \vspace{2mm}
    \includegraphics[width=0.9\textwidth]{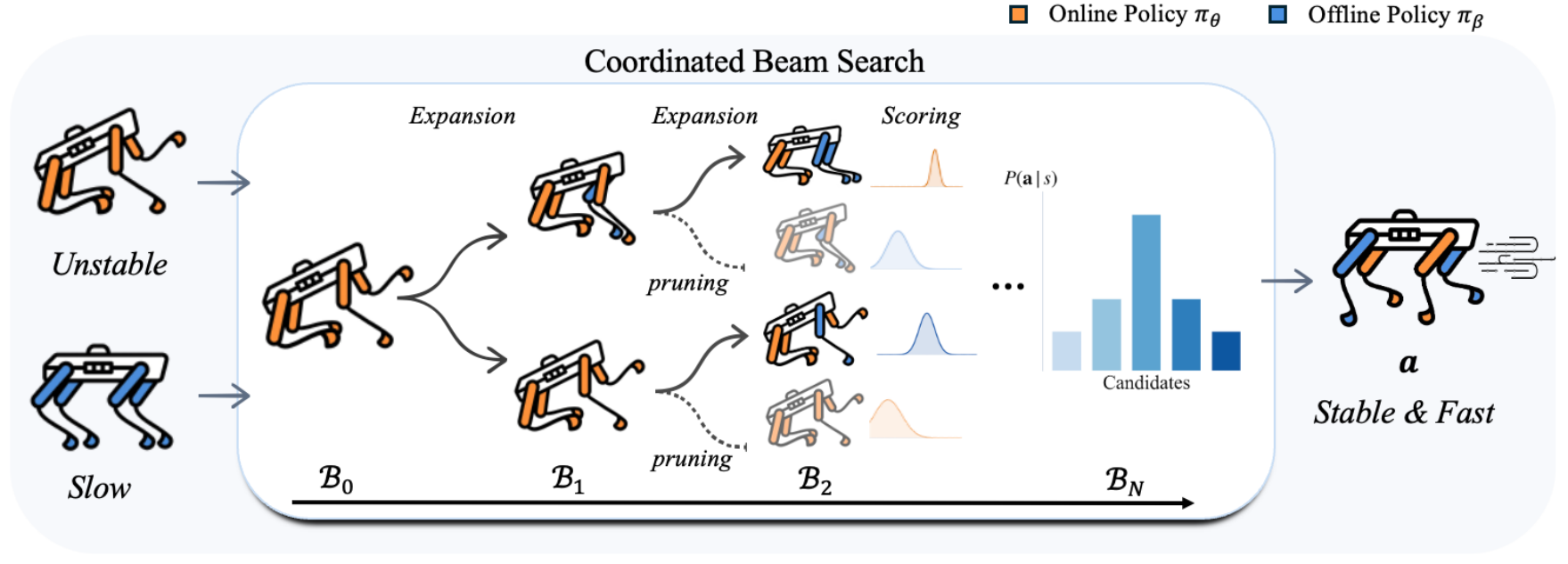}

    \caption{
    Framework of Sim2O.
    Joint actions are constructed through agent-level alternatives from frozen offline policies $\pi_\beta$ and learnable online policies $\pi_\theta$, and selected via \textit{Coordinated Beam Search} using the centralized critic $Q_{\text{tot}}$.}
    \label{fig:framework}
\end{figure*}

\section{Introduction}
\label{sec:intro}

Offline-to-online reinforcement learning has emerged as a pivotal paradigm for facilitating sample-efficient policy optimization in complex decision-making systems~\cite{levine2020offline, fu2020d4rl}. By capitalizing on static datasets for policy initialization and subsequently fine-tuning them via online interaction, this paradigm substantially mitigates the prohibitive costs and safety risks associated with tabula rasa exploration~\cite{nair2020awac, kostrikov2022iql}. Such bootstrapping capabilities are paramount in real-world deployments where physical interaction is financially restrictive, temporally constrained, or safety-critical—rendering purely online learning untenable~\cite{kalashnikov2018qtopt, levine2020offline}.
Within single-agent regimes, offline-to-online adaptation has witnessed extensive maturation. One prominent line of research employs policy constraints or value regularization to anchor the learning trajectory within the support of the offline data distribution, as exemplified by AWAC~\cite{nair2020awac}, IQL~\cite{kostrikov2022iql}, and TD3+BC~\cite{fujimoto2021td3bc}. Conversely, a complementary trajectory adopts an adaptive expansion perspective, dynamically interleaving conservative offline policies with exploratory online behaviors, as explored in PEX~\cite{zhang2023pex} and BOORL~\cite{hu2024boorl}. Despite these algorithmic variances, existing approaches fundamentally share a monolithic abstraction: they conceptualize offline-to-online adaptation as a unified, global decision.

However, directly transplanting single-agent paradigms into the multi-agent domain introduces a fundamental architectural friction, which we refer to as the granularity mismatch. Cooperative multi-agent systems are inherently heterogeneous, characterized by disparate agent roles, localized observations, and asymmetrical learning dynamics. Consequently, enforcing a synchronized, joint-level offline-to-online transition induces a severe coordination bottleneck—either arresting beneficial exploration or destabilizing team cohesion via premature, uncoordinated adaptation.
To resolve this friction, we introduce Sim2O, a conceptually elegant yet highly potent framework tailored for fine-grained multi-agent offline-to-online adaptation. Sim2O empowers individual agents with localized flexibility during candidate action synthesis, while seamlessly preserving centralized, value-based coordination. This structural refinement allows the collective system to dynamically and adaptively balance conservative exploitation and exploratory behaviors during the online fine-tuning phase. Extensive evaluations on the challenging MA-MuJoCo benchmarks~\cite{dewitt2020decentralized} demonstrate that Sim2O consistently surpasses competitive baselines in both learning stability and asymptotic performance. To our knowledge, this work constitutes the first systematic study to analyze and rectify the granularity mismatch in offline-to-online multi-agent reinforcement learning.

Our primary contributions are summarized as follows:
(1) Conceptual Insight: We identify and formalize the granularity mismatch inherent in existing offline-to-online MARL, demonstrating how monolithic, joint-level adaptation stifles coordination in heterogeneous environments.
(2) Algorithmic Framework: We propose Sim2O, a minimalist yet powerful method that facilitates fine-grained, agent-level adaptation flexibility while maintaining rigorous centralized value-based coordination.
(3) Empirical Validation: Through comprehensive evaluations on various complex tasks, we demonstrate that Sim2O significantly improves both training stability and final performance compared to state-of-the-art baselines.

\section{Related Work}
\label{sec:related}
\subsection{Offline Reinforcement Learning}
Offline Reinforcement Learning (RL) addresses the challenge of optimizing decision-making policies exclusively from static, pre-collected datasets without active environment interaction~\cite{levine2020offline, fu2020d4rl,yangopride}. The primary pathology plaguing this regime is \textit{distributional shift}: when the learned policy deviates from the data-generating behavior distribution, bootstrapped value estimates suffer from severe extrapolation errors, leading to the catastrophic overestimation of out-of-distribution (OOD) actions~\cite{levine2020offline, le2019batch, yang2023flow}. To mitigate this vulnerability, existing literature predominantly relies on policy constraints or value regularization. Pioneer frameworks such as BCQ~\cite{fujimoto2019bcq}, BEAR~\cite{kumar2019bear}, BRAC~\cite{wu2019brac}, CRR~\cite{wang2020crr}, and AWR~\cite{peng2019awr} restrict the actor's action support or penalize statistical divergence from the behavior policy. While these techniques yield robust policy initializations under static regimes, they treat the optimization as a closed system, failing to safeguard the highly non-stationary and risk-prone transition to online fine-tuning.

\subsection{Offline Multi-Agent Reinforcement Learning}
Offline Multi-Agent Reinforcement Learning (MARL) complicates this paradigm by compounding standard distributional shifts with joint-action dependencies, exponentially scaling coordination spaces, and the inherent Centralized Training with Decentralized Execution (CTDE) mismatch~\cite{prudencio2023survey, cui2022offline, tseng2022offline, jiang2021offline,yang2026globediff}. Extant methodologies tackle these multi-agent nuances through coordinated policy or value-based regularization. On the policy side, ICQ~\cite{yang2021believe} enforces implicit quantile constraints on decentralized policies to suppress extrapolation errors, while OMAR~\cite{pan2022actorrect} integrates evolutionary optimization to rectify joint policy drift. On the value side, MACQL~\cite{shao2023counterfactual} extends conservative value iteration by penalizing joint state-action value functions. More recently, frameworks like OMIGA~\cite{wang2023globaltolocal} adapt structured value-decomposition architectures—such as QMIX~\cite{rashid2018qmix} and QPLEX~\cite{son2021qplex}—to offline settings, leveraging implicit global-to-local regularization to reconcile centralized coordination with decentralized execution. Crucially, while these methods establish coherent coordination under static data, they are structurally rigid and lack adaptive mechanisms to accommodate agent-level variability during online deployment.

\subsection{Offline-to-Online Adaptation}
The bridge between offline initialization and active online refinement has witnessed rapid development, though its deployment remains largely restricted to single-agent settings. Existing literature can be broadly dichotomized into two paradigms based on how they leverage pre-trained offline priors:

\paragraph{Conservative Fine-Tuning and Regularization} This paradigm treats the pre-trained offline policy as a warm-start initialization while imposing behavioral or value constraints during online interactions to ward off catastrophic value collapse. Representative techniques utilize advantage-weighted policy updates~\cite{nair2020awac, fujimoto2021td3bc}, implicit value regularization~\cite{kostrikov2022iql, nakamoto2023calql}, or adaptive epistemic uncertainty quantification~\cite{li2023proto}. To maximize sample efficiency, data-centric variants like RLPD~\cite{ball2023rlpd} pair aggressive replay buffer blending with structural normalization. Despite their algorithmic diversity, these methods share a unified assumption: they focus on stabilizing the continuous parameter evolution of a single, monolithic policy.

\paragraph{Structural Policy Expansion and Composition} Rather than forcibly updating a singular network, this alternative paradigm explicitly preserves the offline policy as an unalterable or distinct functional anchor. For instance, PEX~\cite{zhang2023pex} and BOORL~\cite{hu2024boorl} construct composite policy architectures that dynamically arbitrate between conservative offline behaviors and exploratory online policies. This explicit decoupling allows the agent to safely venture into novel state-action spaces without discarding the foundational competence acquired offline.

\section{Preliminaries}
\subsection{Dec-POMDP Formulation}
We formulate the cooperative multi-agent task within the mathematical framework of a Decentralized Partially Observable Markov Decision Process (Dec-POMDP), which is formally defined by the tuple $\mathcal{M} = \langle S, \mathcal{A}, P, R, N, \Omega, O, \gamma \rangle$. 
Here, $S$ denotes the global state space, $\mathcal{A} = \prod_{i=1}^N \mathcal{A}^i$ represents the joint action space, and $N$ signifies the total number of agents. 
At each environmental timestep $t$, each individual agent $i \in \{1, \dots, N\}$ receives a local observation $o^i \in \Omega$ generated by the observation function $O^i(s)$, and independently selects an individual action $a^i \in \mathcal{A}^i$. 
The aggregation of these individual choices constitutes the joint action vector $\bm{a} = \langle a^1, \dots, a^N \rangle \in \mathcal{A}$, which prompts the environment to transition to the next state $s' \sim P(\cdot \mid s, \bm{a})$ and yields a shared global reward $r = R(s, \bm{a})$. 
The collective objective of the team of agents is to parameterize a joint policy $\bm{\pi}$ that maximizes the expected discounted global return $J(\bm{\pi}) = \mathbb{E} \left[ \sum_{t=0}^\infty \gamma^t r_t \right]$, where $\gamma \in [0, 1)$ serves as the temporal discount factor.

\subsection{Offline-to-Online Reinforcement Learning}
\label{sec:prelim_offline-to-online}
Offline-to-online reinforcement learning investigates the algorithmic paradigm where an agent transitions seamlessly from non-interactive training on static datasets to active online policy improvement. 
Formally, given a static offline dataset $\mathcal{D}_{\text{off}}$, the initial offline phase constructs a behavioral policy $\bm{\pi}_{\beta}$ and a centralized action-value critic $Q_{\text{tot}}$ through statistical learning. 

To bridge the operational gap prior to active deployment, the online adaptation stage initializes the online policy parameters $\bm{\pi}_{\theta}$ using the learned weights of $\bm{\pi}_{\beta}$ and inherits the structure of $Q_{\text{tot}}$ to exploit the historical value estimates acquired from the static data. 
During the subsequent online fine-tuning phase, the joint policy interacts dynamically with the environment to collect fresh transitions, which are stored in an online experience replay buffer $\mathcal{D}_{\text{on}}$. 
The core objective is to continuously refine the online policy $\bm{\pi}_{\theta}$ and update the centralized critic $Q_{\text{tot}}$ by simultaneously utilizing the structural priors from $\mathcal{D}_{\text{off}}$ and the exploratory trajectories from $\mathcal{D}_{\text{on}}$. 
The fundamental challenge underlying this algorithmic transition resides in mitigating the severe distribution shift between the static offline data distribution and the evolving online visitation distribution.

\section{Methodology}
\label{sec:method}

Existing single-agent offline-to-online methods typically formulate the transition from offline priors to online exploration as a unified decision-making process. 
However, directly extending this paradigm to multi-agent systems introduces a severe coordination bottleneck, as forcing all agents to strictly and simultaneously adhere to either offline or online modes severely restricts the behavioral flexibility required to discover superior joint policies.
To resolve this dilemma, we introduce Sim2O, a streamlined compositional framework for offline-to-online MARL that decouples individual action proposals while maintaining global coordination.

The remainder of this section is organized as follows. First, Section~\ref{sec:hybrid_space} formally defines the online adaptation optimization problem within a fine-grained, hybrid joint action space. Subsequently, Section~\ref{sec:cbs} introduces Coordinated Beam Search, a novel algorithmic component designed to solve this problem by incrementally constructing high-quality joint actions through agent-wise expansion and centralized evaluation. Finally, Section~\ref{sec:training} details the complete training protocol and practical implementation considerations. The overall architecture of Sim2O is schematically illustrated in Figure~\ref{fig:framework} and formalized in Algorithm~\ref{alg:sim2o}.


\subsection{Fine-Grained Adaptation Formulation}
\label{sec:hybrid_space}

To formalize the offline-to-online adaptation process, we first distinguish between two distinct operational scales in multi-agent environments: \textit{joint-level} and \textit{agent-level} decision-making paradigms. 
Under a conventional joint-level paradigm, the executed joint action must be drawn entirely and monolithically from either the offline policy or the active online policy. 
This coarse-grained coupling enforces an all-or-nothing constraint on action selection, severely restricting behavioral flexibility during online fine-tuning. 
We term this structural limitation \textit{granularity mismatch}. 

In contrast, an agent-level paradigm allows candidate joint actions to be flexibly constructed by selectively pooling offline and online action proposals across different agents. 
Formally, given the current state $s$, for each agent $i \in \{1, \dots, N\}$, we define a local action candidate set $\mathcal{A}^i = \{a^i_\beta, a^i_\theta\}$, where $a^i_\beta \sim \pi^i_\beta(\cdot|o^i)$ and $a^i_\theta \sim \pi^i_\theta(\cdot|o^i)$ represent the actions proposed by the frozen offline prior and the learnable online policy, respectively. 
A fine-grained \textit{hybrid joint action} $\bm{a}$ is then formed by selecting exactly one alternative action for each agent, yielding a combinatorial hybrid action space:
\begin{equation}
    \mathcal{A}_{\text{hyb}} = \prod_{i=1}^{N} \mathcal{A}^i.
\end{equation}

Leveraging this fine-grained joint action space, we formulate the optimal online adaptation policy as identifying the specific hybrid joint action that maximizes the global value estimate. Consequently, the optimal joint action $\bm{a}^*$ is derived by solving:
\begin{equation}
    \bm{a}^* = \mathop{\arg\max}_{\bm{a} \in \mathcal{A}_{\text{hyb}}} Q_{\text{tot}}(s, \bm{a}).
    \label{eq:objective}
\end{equation}

Eq.~\ref{eq:objective} formalizes the foundational principle of Sim2O: \emph{local expansion with centralized evaluation}. Crucially, while individual action candidates are generated independently at the agent level, the composition of the final hybrid joint action is systematically guided by the centralized value function to guarantee coherent team coordination.

This formulation offers a distinct advantage over synchronized adaptation: it dynamically capitalizes on localized behavioral improvements in the online policy while preserving the structural safeguards of offline priors. 
However, directly optimizing Eq.~\ref{eq:objective} via exhaustive search is computationally intractable, as the cardinality of $\mathcal{A}_{\text{hyb}}$ scales exponentially with the number of agents as $\mathcal{O}(2^N)$. To bypass this computational bottleneck, we introduce a structured, linear-time inference procedure in the following subsection.

\subsection{Coordinated Beam Search}
\label{sec:cbs}

To efficiently solve Eq.~\ref{eq:objective} without explicit enumeration, we introduce Coordinated Beam Search (CBS), an algorithmic component with linear complexity that incrementally constructs high-quality joint actions through iterative agent-wise expansion and centralized evaluation. 
The operational mechanics of CBS proceed through the sequential phases detailed below.

\textbf{Initialization.} 
The backbone of the CBS algorithm is the beam $\mathcal{B}$, a tracking structure maintained to store the $k$ (beam width) most promising partial or full joint action candidates identified during the search. 
The search is anchored by initializing the beam with the baseline joint action generated exclusively by the online policy:
\begin{equation}
\label{eq:beam_init}
    \mathcal{B}_0 = \{\bm{a}_{\theta}\}, \quad \text{where } \bm{a}_{\theta} = \langle a^1_\theta, \dots, a^N_\theta \rangle \text{ with } a^i_\theta \sim \pi^i_\theta(\cdot|o^i).
\end{equation}
This default initialization establishes a high-quality, coordinated baseline for subsequent agent-wise iterative refinement.

\begin{figure}[t] 
\begin{minipage}{1.0\linewidth}
\vspace*{1pt}
\begin{algorithm}[H]
\caption{Sim2O: Efficient Offline-to-Online MARL via Joint Action Composition}
\label{alg:sim2o}

\begin{algorithmic}
\REQUIRE Offline dataset $\mathcal{D}_{\text{off}}$, offline policy $\pi_\beta$
\STATE Initialize online policy $\pi_\theta$ and value functions
\FOR{$t=0$ {\bfseries to} $T-1$}
    \STATE Initialize $\mathcal{B}_0$  based on Eq.~\ref{eq:beam_init}
    \FOR{$j \in \textsc{RandomPermutation}(\{1,\dots,N\})$}
        \STATE $\mathcal{C}_j \leftarrow \mathcal{B}_{j-1}$ based on Eq.~\ref{eq:expansion}
        \STATE $\mathcal{B}_j \leftarrow \mathcal{C}_j $ based on Eq.~\ref{eq:pruning2}
    \ENDFOR
    \STATE $\bm{a}_{t} \sim \mathcal{B}_N$ based on Eq.~\ref{eq:pruning}
    \STATE Execute $\bm{a}_{t}$ in the Env and add transitions into $  \mathcal{D}_{\text{on}}$
    \STATE Combine $\mathcal{D}_{\text{off}}$ and $  \mathcal{D}_{\text{on}}$ to obtain $\mathcal{D}$
    \STATE $\phi \leftarrow \phi - \eta \nabla_{\phi} \mathcal{L}_{V}(\phi)$ based on Eq.~\ref{eq:loss_v}
    \STATE $\psi \leftarrow \psi - \eta \nabla_{\psi} \mathcal{L}_{Q}(\psi)$ based on Eq.~\ref{eq:loss_q}
    \STATE $\theta \leftarrow \theta - \eta \nabla_{\theta} \mathcal{L}_{\pi}(\theta)$ based on Eq.~\ref{eq:loss_pi}
\ENDFOR
\end{algorithmic}
\end{algorithm}
\end{minipage}
\end{figure}

\textbf{Agent-wise Expansion and Evaluation.} 
To systematically explore the hybrid combinatorics, we iteratively expand the search space by branching out from the profiles preserved in the current beam. 
Specifically, during the $j$-th iteration, we target a specific agent $i \in \{1, \dots, N\}$ (e.g., following a sequential order or a random permutation). For every candidate joint action $\bm{a} \in \mathcal{B}_{j-1}$ currently held in the beam, we construct a hybrid variant by replacing its $i$-th component with the offline action prior $a^i_\beta \sim \pi^i_\beta(\cdot|o^i)$. 
The expanded candidate set $\mathcal{C}_j$ is then formally defined as the union of the historical beam and these newly generated counterfactual variants:
\begin{equation}
\label{eq:expansion}
    \mathcal{C}_j = \mathcal{B}_{j-1} \cup \left\{ \langle a^i_\beta, \bm{a}^{-i} \rangle \mid \forall \bm{a} \in \mathcal{B}_{j-1} \right\},
\end{equation}
where $\bm{a}^{-i}$ denotes the sub-vector of actions for all agents excluding agent $i$. 
Consequently, the cardinality of the candidate set precisely doubles at each expansion step, i.e., $|\mathcal{C}_j| = 2|\mathcal{B}_{j-1}|$. By seamlessly cross-breeding the behavioral footprints of the online beam with the offline priors, $\mathcal{C}_j$ encompasses a highly diverse spectrum of hybrid joint action profiles.

\textbf{Probabilistic Pruning and Selection.}
Absent a capacity-bounding mechanism, the cardinality of the candidate set $\mathcal{C}_j$ would compound exponentially at a rate of $\mathcal{O}(2^j)$, quickly breaching the computational tractability established by the beam width $k$. 
To maintain algorithmic efficiency while preserving high-reward exploratory paths, we implement a probabilistic pruning mechanism to dynamically filter the search space. 
We first evaluate all candidate profiles in $\mathcal{C}_j$ using the centralized action-value function $Q_{\text{tot}}$, and establish a soft-max distribution over the candidate set:
\begin{equation}
    \label{eq:pruning}
    P(\bm{a} \mid s) = \frac{\exp\left(Q_{\text{tot}}(s, \bm{a})/\tau\right)}{\sum_{\bm{a}' \in \mathcal{C}_j} \exp\left(Q_{\text{tot}}(s, \bm{a}')/\tau\right)},
\end{equation}
where $\tau > 0$ represents a temperature parameter regulating the selection intensity. The updated beam $\mathcal{B}_j$ is subsequently constructed by sampling $k$ unique joint actions from $\mathcal{C}_j$ without replacement, governed by the evaluated probabilities:
\begin{equation}
\label{eq:pruning2}
    \mathcal{B}_j = \operatorname{Sample-}k_{\bm{a} \in \mathcal{C}_j} \{ P(\bm{a}\mid s) \}.
\end{equation}
This probabilistic filtration ensures that joint actions demonstrating superior global coordination and value estimates are preferentially retained, while unpromising branches are gracefully discarded.

This dual-phase cycle of agent-wise expansion and probabilistic pruning is executed sequentially for $N$ iterations to yield the final converged beam $\mathcal{B}_N$. Ultimately, the definitive joint action to be deployed in the environment is sampled from $\mathcal{B}_N$ according to the distribution defined in Eq.~\ref{eq:pruning}.

    

\begin{figure*}[t]
    \centering
    \vspace{2mm}

    \subfloat[HalfCheetah 6x1]{%
        \includegraphics[width=0.18\textwidth]{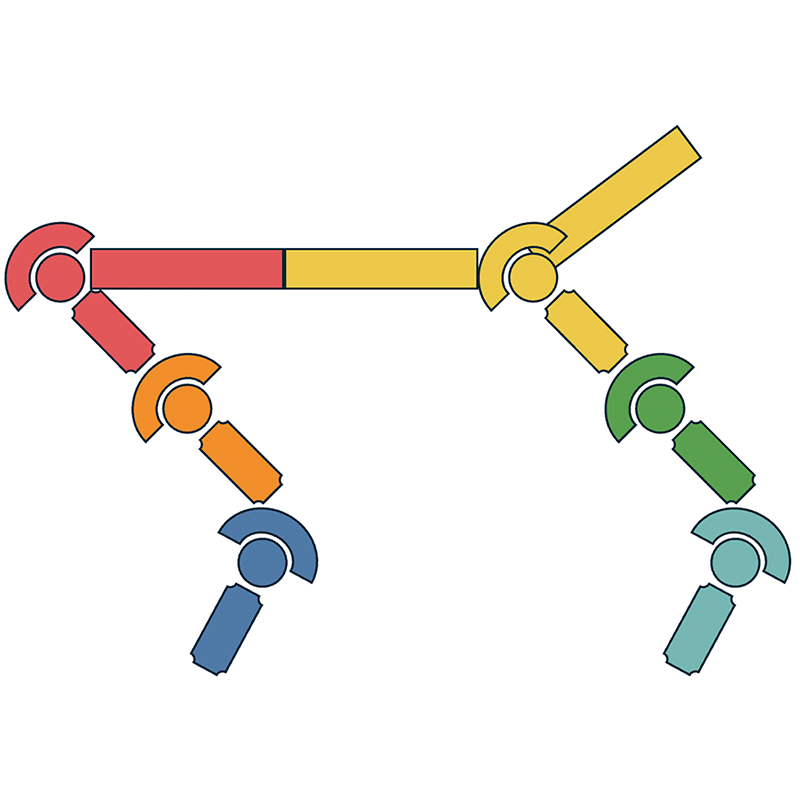}
        \label{fig:cheetah}
    }
    \hfill
    \subfloat[Ant 8x1]{%
        \includegraphics[width=0.18\textwidth]{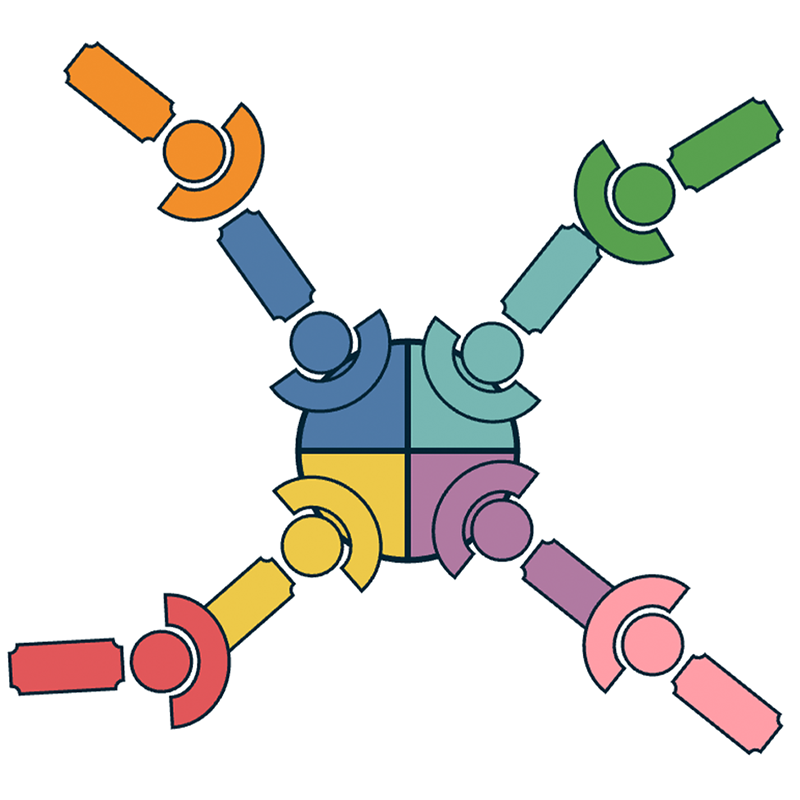}
        \label{fig:ant8x1}
    }
    \hfill
    \subfloat[Ant 2x4]{%
        \includegraphics[width=0.18\textwidth]{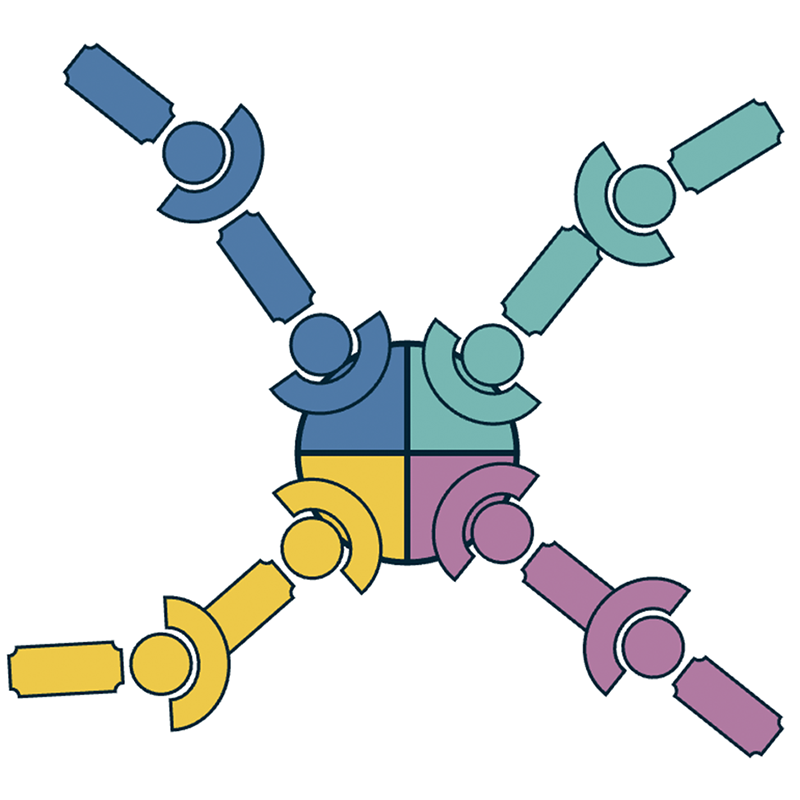}
        \label{fig:ant2x4}
    }
    \hfill
    \subfloat[Hopper 3x1]{%
        \includegraphics[width=0.18\textwidth]{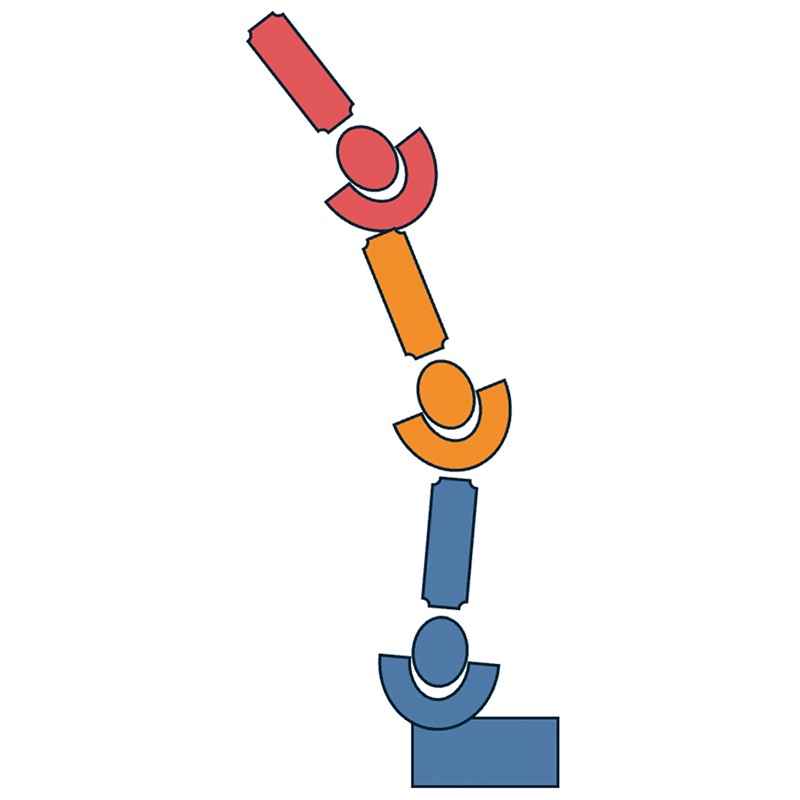}
        \label{fig:hopper}
    }

    \caption{Agent partitionings for MA-MuJoCo environments. Different colors represent different agents controlling specific joints. (a) HalfCheetah with 6 agents. (b) Ant with 8 agents. (c) Ant with 4 agents. (d) Hopper with 3 agents.}
    \label{fig:mamujoco_viz}
\end{figure*}

\subsection{Practical Implementation}
\label{sec:training}
The Sim2O framework operates as a general, modular paradigm that remains agnostic to the underlying MARL algorithm. 
In this work, we instantiate our framework using OMIGA~\cite{wang2023globaltolocal} as the primary algorithmic backbone. Following the standard offline-to-online adaptation protocol, we pre-train the decentralized policies and the centralized critic on the static offline dataset $\mathcal{D}_{\text{off}}$. 
Subsequently, the parameters of the active online policy $\bm{\pi}_\theta$ and the centralized critic $Q_{\text{tot}}$ are initialized with these pre-trained weights, while the offline prior policy $\bm{\pi}_\beta$ remains frozen throughout the online phase to act as a stable behavioral anchor.

During the online fine-tuning phase, at each environmental time step $t$, Sim2O generates an executable joint action $\bm{a}_t$ via the Coordinated Beam Search procedure detailed in Section~\ref{sec:cbs}. 
The resulting environmental transitions are collected and stored in an online replay buffer $\mathcal{D}_{\text{on}}$. 
To mitigate catastrophic forgetting and stabilize the online optimization trajectory, we employ a balanced experience replay mechanism, wherein training minibatches $\mathcal{D}$ are sampled from a joint mixture of $\mathcal{D}_{\text{off}}$ and $\mathcal{D}_{\text{on}}$ governed by a fixed mixing ratio $\rho$.

Model parameters are optimized by applying the core OMIGA objectives~\cite{pan2022actorrect} over the sampled minibatches $\mathcal{D}$. Specifically, the local value functions $V^i(o^i)$, parameterized by $\phi$, are updated by minimizing the following objective:
\begin{equation}
\label{eq:loss_v}
\begin{aligned}
\mathcal{L}_{V}(\phi)
= \mathbb{E}_{\mathcal{D}} \bigg[
    \sum_{i=1}^N
    \bigg(
        \exp &\left(
            \frac{w^i(s)}{\alpha}
            \big( Q^i(o^i, a^i) - V^i(o^i) \big)
        \right) \\
        &+ \frac{w^i(s) V^i(o^i)}{\alpha}
    \bigg)
\bigg],
\end{aligned}
\end{equation}
where $w^i(s)$ denotes the state-dependent agent coordination weights, and $\alpha$ represents the temperature regularization coefficient. 
Concurrently, the centralized critic $Q_{\text{tot}}$, parameterized by $\psi$, is optimized via the temporal-difference loss:
\begin{equation}
\label{eq:loss_q}
\mathcal{L}_{Q}(\psi)
= \mathbb{E}_{\mathcal{D}} \left[
    \left( r + \gamma V_{\text{tot}}(s') - Q_{\text{tot}}(s, \bm{a}) \right)^2
\right],
\end{equation}
where the global value and action-value functions are factorized as $V_{\text{tot}}(s) = \sum_{i=1}^N w^i(s) V^i(o^i) + b(s)$ and $Q_{\text{tot}}(s, \bm{a}) = \sum_{i=1}^N w^i(s) Q^i(o^i, a^i) + b(s)$, respectively, with $b(\cdot)$ representing a state-dependent offset network. 
ly, the parameters $\theta$ of the online policy are updated by maximizing the policy advantage profile:
\begin{equation}
\label{eq:loss_pi}
\begin{aligned}
\mathcal{L}_{\pi}(\theta)
= - \mathbb{E}_{\mathcal{D}} \bigg[
    \sum_{i=1}^N
    &\exp \left(
        \frac{w^i(s)}{\alpha}
        \big( Q^i(o^i, a^i) - V^i(o^i) \big)
    \right) \\
    &\cdot \log \pi_{\theta}^i(a^i \mid o^i)
\bigg].
\end{aligned}
\end{equation}
The complete pipeline of the offline-to-online adaptation procedure is structured and summarized in Algorithm~\ref{alg:sim2o}.

\subsection{Analysis}
\label{sec:analysis}

\paragraph{Computational Complexity.}
As formulated in Section~\ref{sec:hybrid_space}, determining the optimal adaptation policy mandates maximizing Eq.~\ref{eq:objective} over the combinatorial hybrid joint action space $\mathcal{A}_{\text{hyb}}$. Since the cardinality of this space scales exponentially as $|\mathcal{A}_{\text{hyb}}| = 2^N$, an exhaustive enumeration incurs a prohibitive computational overhead of $\mathcal{O}(2^N)$, rendering direct optimization intractable for environments with a large scaling count of agents. Coordinated Beam Search (CBS) systematically alleviates this tractability bottleneck by decomposing the global joint optimization into a sequential pathway of $N$ localized, agent-wise expansions. In each sequential iteration $j$, the candidate set $\mathcal{C}_j$ undergoes probabilistic pruning, ensuring that the cardinality of the active beam $\mathcal{B}_{j}$ is strictly bounded by the configured beam width $k$ (i.e., $|\mathcal{B}_j| \le k$). Consequently, the overall computational complexity per decision step is bounded by $\mathcal{O}(N k)$ forward passes of the centralized critic network $Q_{\text{tot}}$. This reduction achieves a critical scaling paradigm shift: transitioning the computational footprint from an intractable exponential cost to a linear complexity with respect to the number of agents $N$.

\paragraph{Impact of Beam Width.}
The beam width $k$ serves as a pivotal hyperparameter that mediates the intrinsic trade-off between search coverage and computational efficiency. In the strict edge case where $k = 1$, CBS degenerates into a purely greedy, single-candidate local search. While computationally minimal, this regime lacks backtracking capabilities and is highly susceptible to becoming trapped in sub-optimal local extrema due to the highly non-convex nature of the cooperative joint $Q$-value landscape. Conversely, as $k$ scales upward, Sim2O preserves a more diverse ensemble of intermediate hybrid joint action candidates. This behavioral diversity effectively expands the search horizon, thereby empowering CBS to robustly navigate the combinatorial intersections of offline priors and active online exploration, ultimately yielding approximations that lie significantly closer to the true global optimum.

\section{Theoretical Analysis}
\label{sec:theory}
In this section, we clear a rigorous theoretical pathway to analyze the approximation gap between coordinated beam search and the global optimal solution within the combinatorial hybrid joint action space. 
We first define a local individual candidate action set for each agent $i$ at the current decision step as $\mathcal{A}^i = \{a^i_\beta, a^i_\theta\}$, which explicitly comprises the discrete individual actions generated by the frozen offline behavior policy $a^i_\beta \sim \pi^i_\beta(\cdot \mid o^i)$ and the active online policy $a^i_\theta \sim \pi^i_\theta(\cdot \mid o^i)$, respectively. 
Based on these localized components, we formalize two distinct archetypes of joint action spaces that govern the online adaptation process.

\begin{definition}[Synchronized Joint Action Space]
\label{def:sync_space}
The synchronized joint action space constrained within traditional online adaptation paradigms is defined as the binary set of fully aligned joint action vectors:
\begin{equation}
\mathcal{A}_{\text{sync}}
=
\big\{ \bm{a}_\beta, \bm{a}_\theta \big\},
\end{equation}
where $\bm{a}_\beta = \langle a^1_\beta, \dots, a^N_\beta \rangle$, $\bm{a}_\theta = \langle a^1_\theta, \dots, a^N_\theta \rangle$, and the cardinality of this restricted space is strictly bounded as $|\mathcal{A}_{\text{sync}}|=2$.
\end{definition}

\begin{definition}[Fine-Grained Hybrid Action Space]
\label{def:fine_space}
The fine-grained hybrid joint action space is structured as the complete Cartesian product of the localized individual alternative sets:
\begin{equation}
    \mathcal{A}_{\text{hyb}} = \prod_{i=1}^{N} \mathcal{A}^i,
\end{equation}
where the exponential combination yields a cardinality of $|\mathcal{A}_{\text{hyb}}|=2^N$.
\end{definition}

\begin{theorem}[Search Space Dominance]
\label{thm:dominance}
For the true global joint action-value function $Q_{\text{tot}}(s,\bm{a})$, the optimization upper bound satisfies:
\begin{equation}
\max_{\bm{a}\in\mathcal{A}_{\text{hyb}}} Q_{\text{tot}}(s,\bm{a})
\;\ge\;
\max_{\bm{a}\in\mathcal{A}_{\text{sync}}} Q_{\text{tot}}(s,\bm{a}).
\end{equation}
\end{theorem}
\begin{proof}
    Please refer to Appendix~\ref{app:proof_dominance} for the detailed mathematical proof.
\end{proof}

Theorem~\ref{thm:dominance} theoretically validates that traditional online adaptation confined within the Synchronized Joint Action Space is inherently subject to strict performance ceilings. 
Conversely, the Fine-Grained Hybrid Action Space, enabled by agent-level combinatorial expansion, introduces an augmented spectrum of hybrid joint action profiles that possess potentially superior global value estimates.

\begin{figure*}[t]
    \centering
    \includegraphics[width=0.75\textwidth]{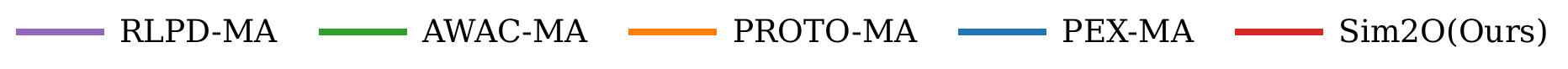}

    \subfloat[Ant 2x4 Medium]{%
        \includegraphics[width=0.24\textwidth]{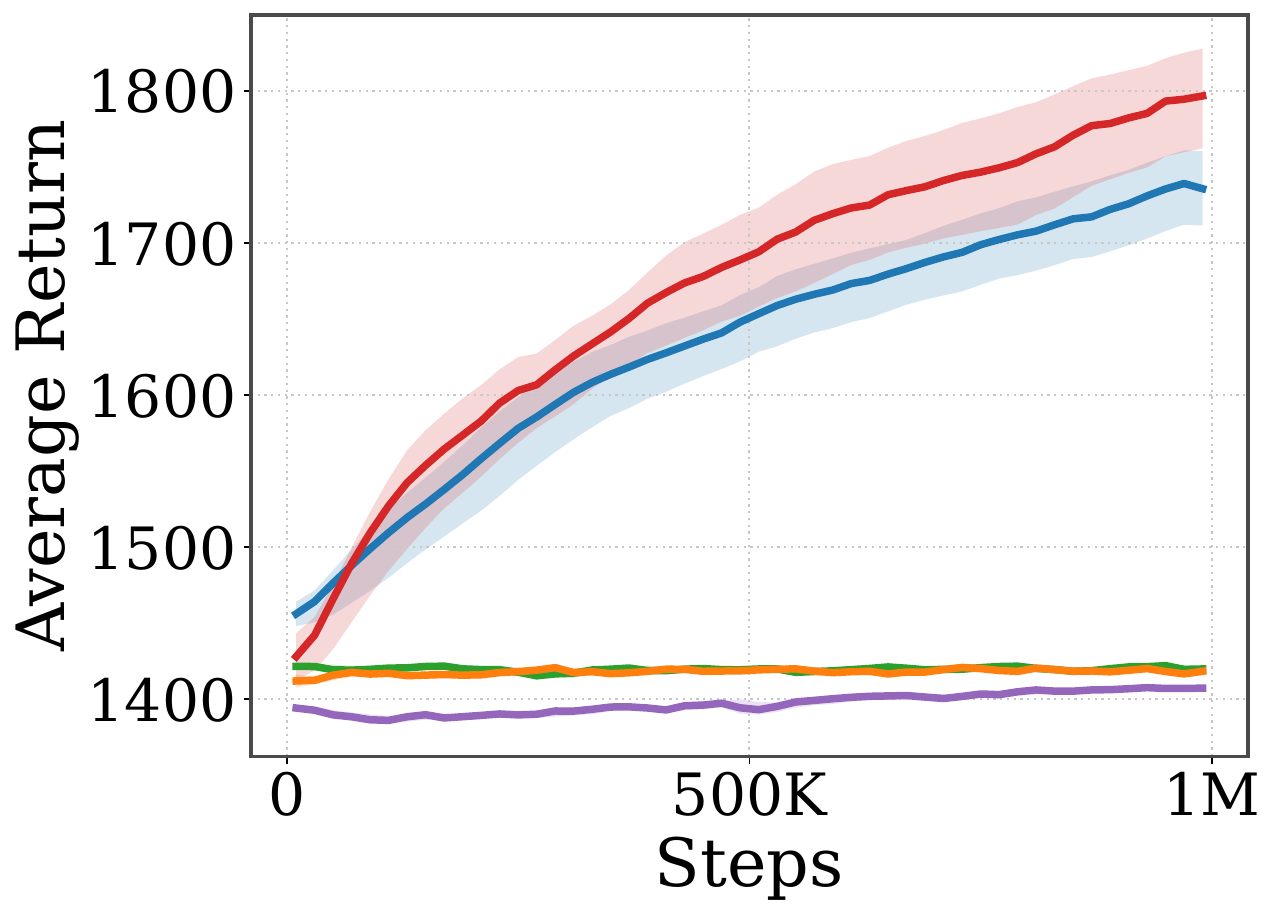}
    }
    \hfill
    \subfloat[Ant 2x4 Med-Rep]{%
        \includegraphics[width=0.24\textwidth]{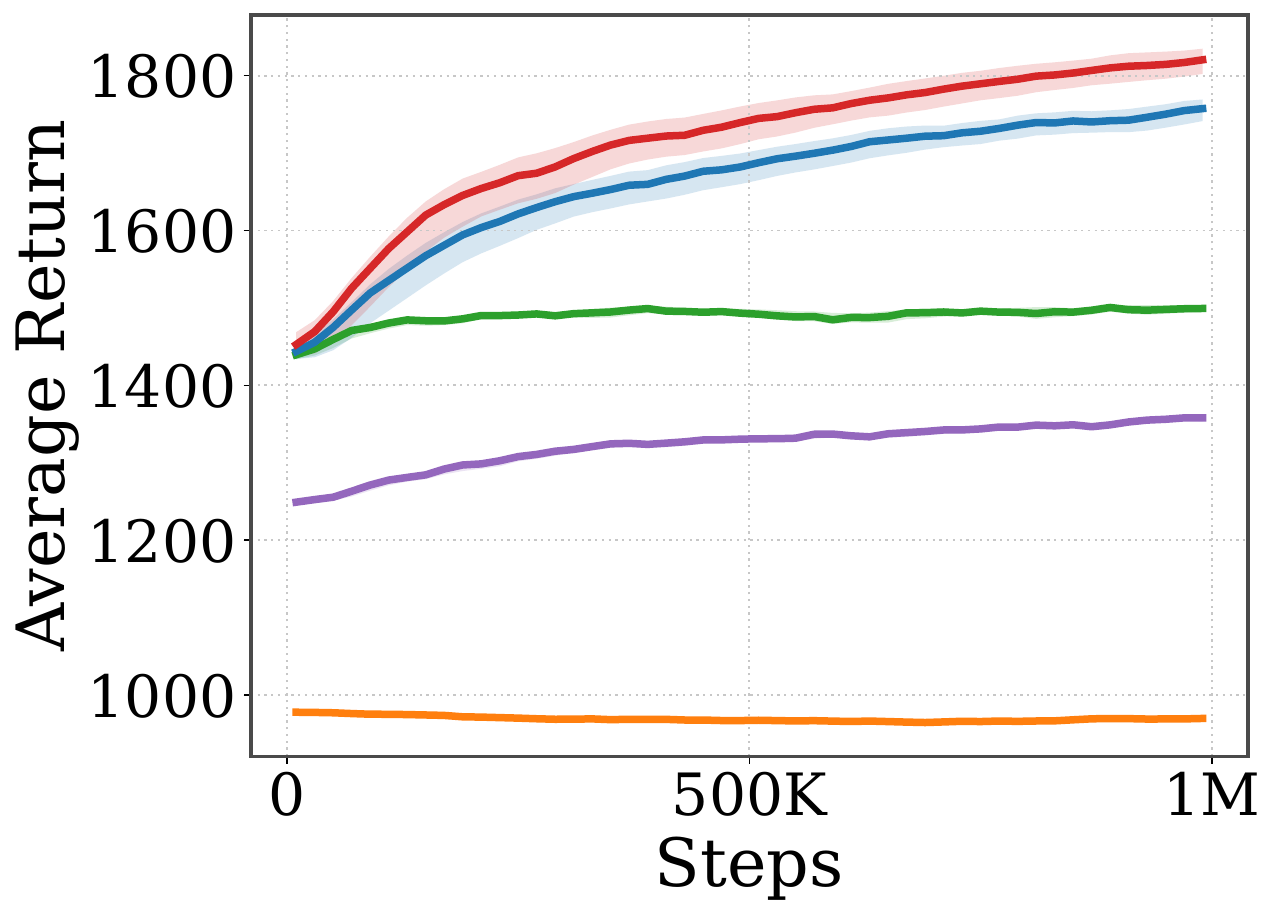}
    }
    \hfill
    \subfloat[Ant 4x2 Medium]{%
        \includegraphics[width=0.24\textwidth]{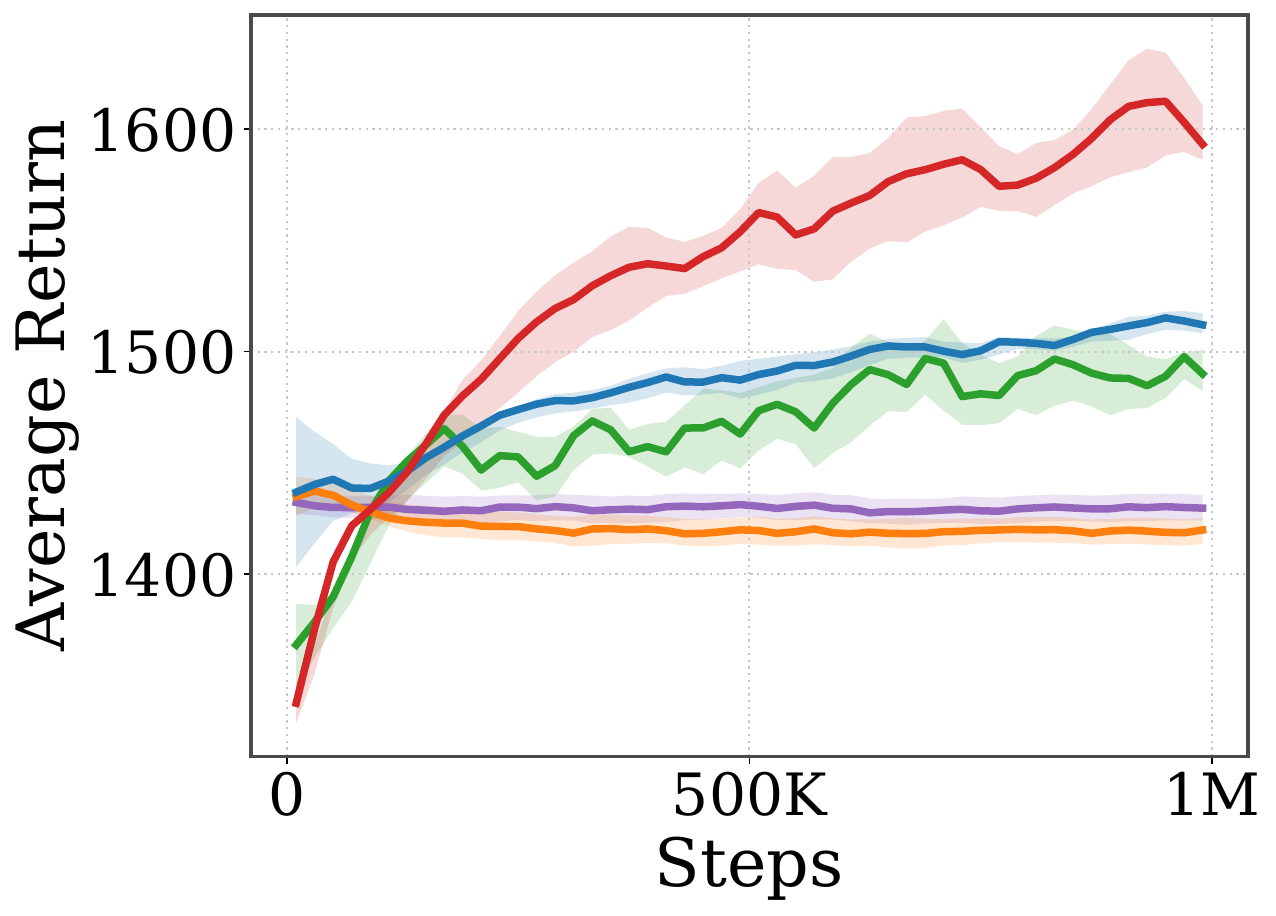}
    }
    \hfill
    \subfloat[Ant 4x2 Med-Exp]{%
        \includegraphics[width=0.24\textwidth]{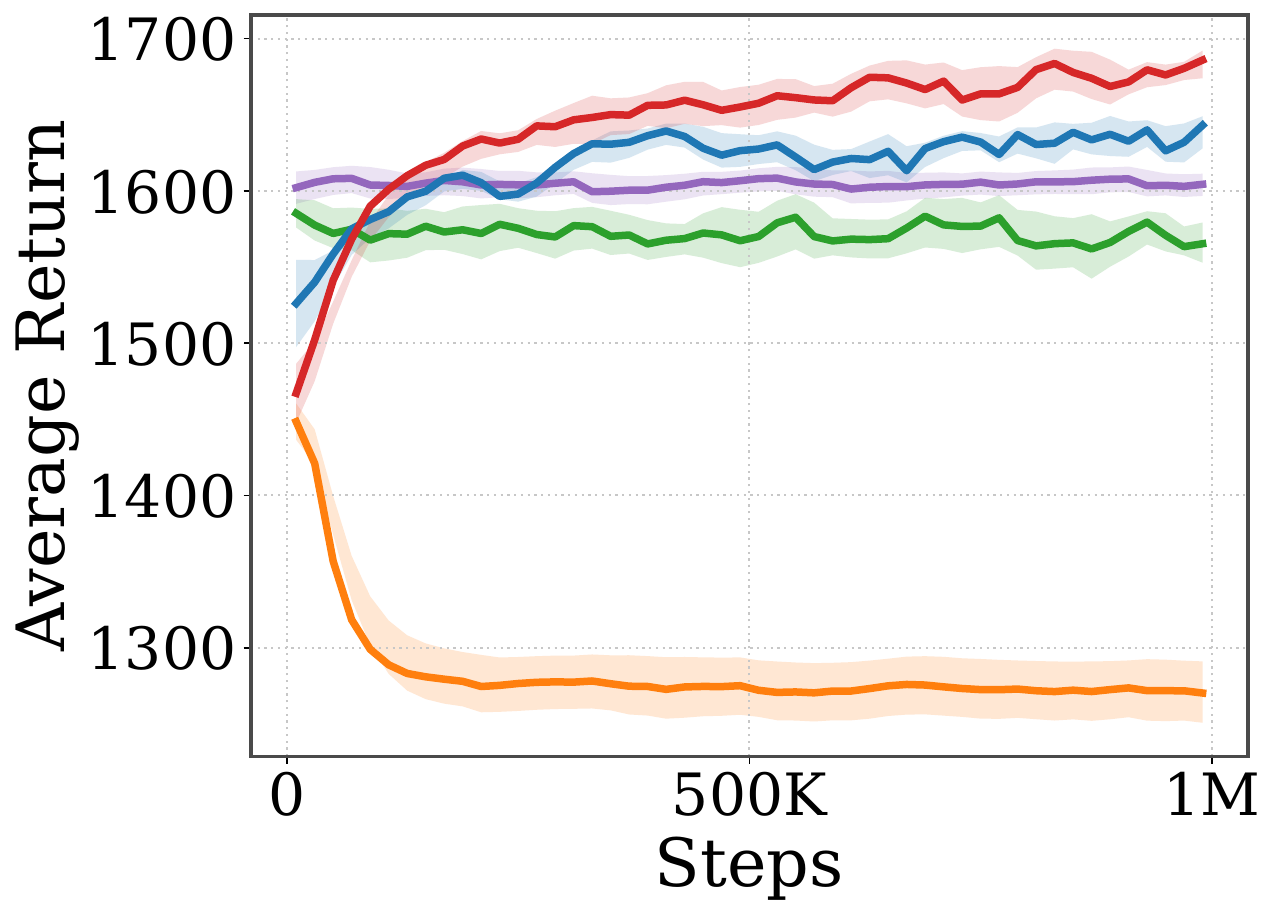}
    }

    \subfloat[Ant 4x2 Med-Rep]{%
        \includegraphics[width=0.24\textwidth]{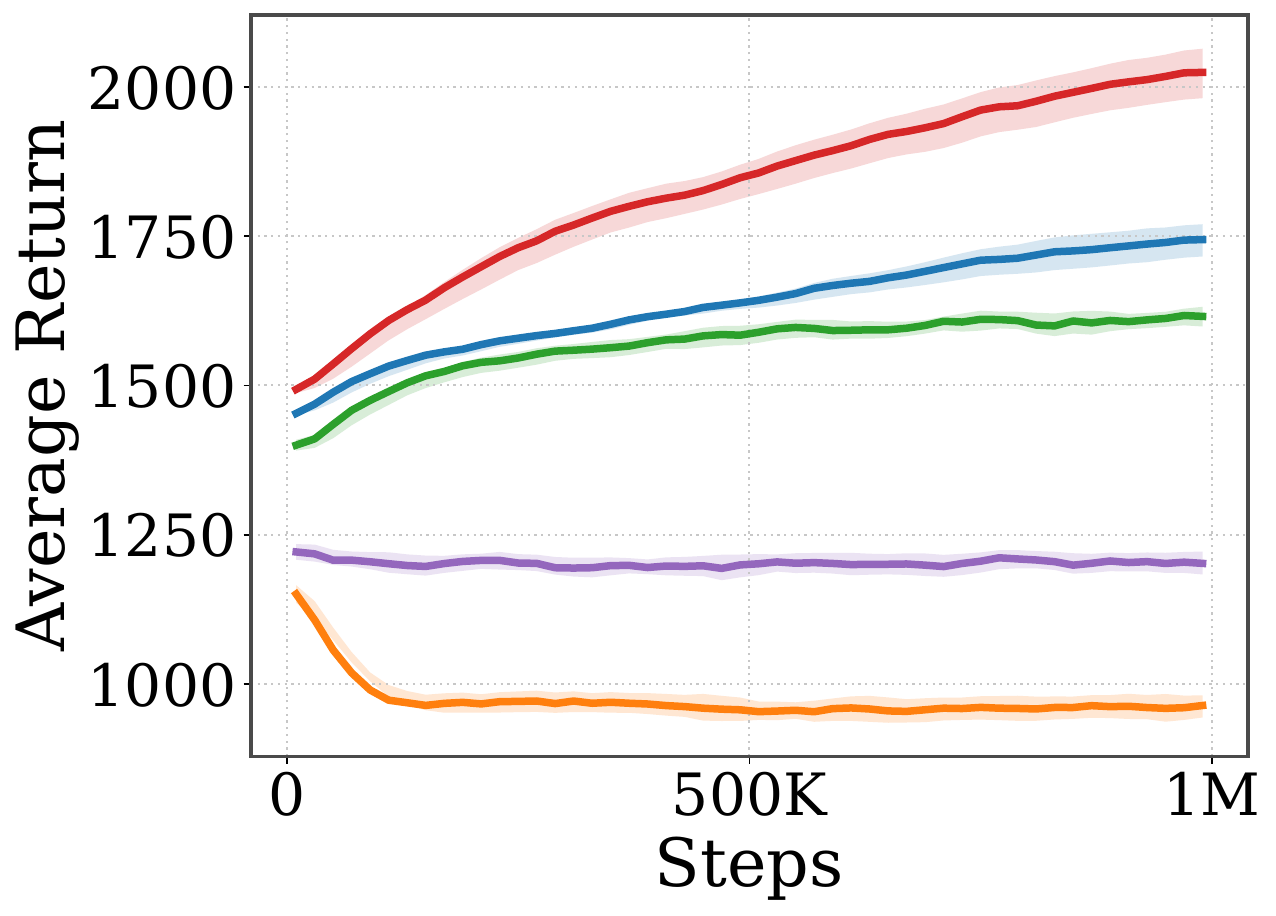}
    }
    \hfill
    \subfloat[Ant 8x1 Med-Rep]{%
        \includegraphics[width=0.24\textwidth]{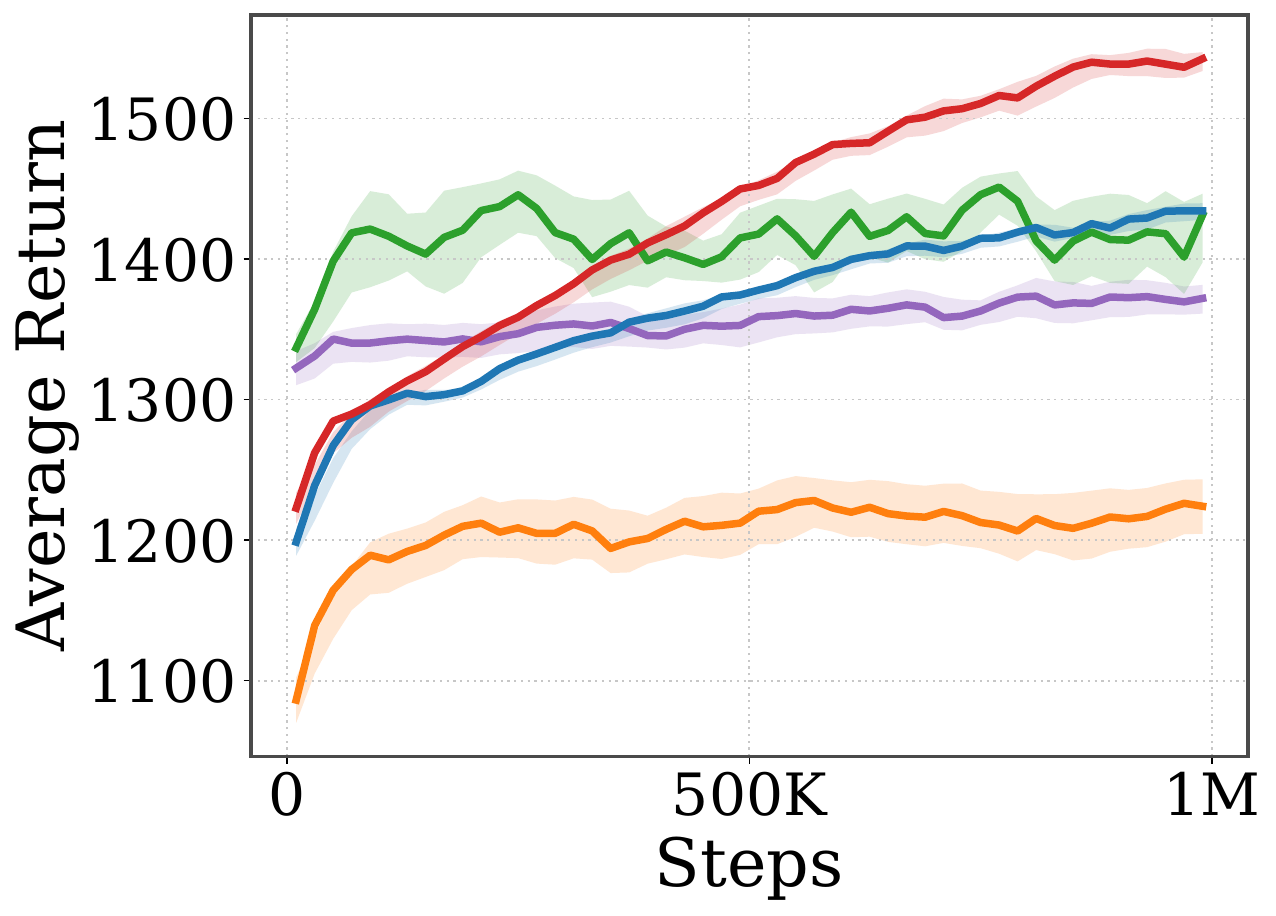}
    }
    \hfill
    \subfloat[HC 6x1 Medium]{%
        \includegraphics[width=0.24\textwidth]{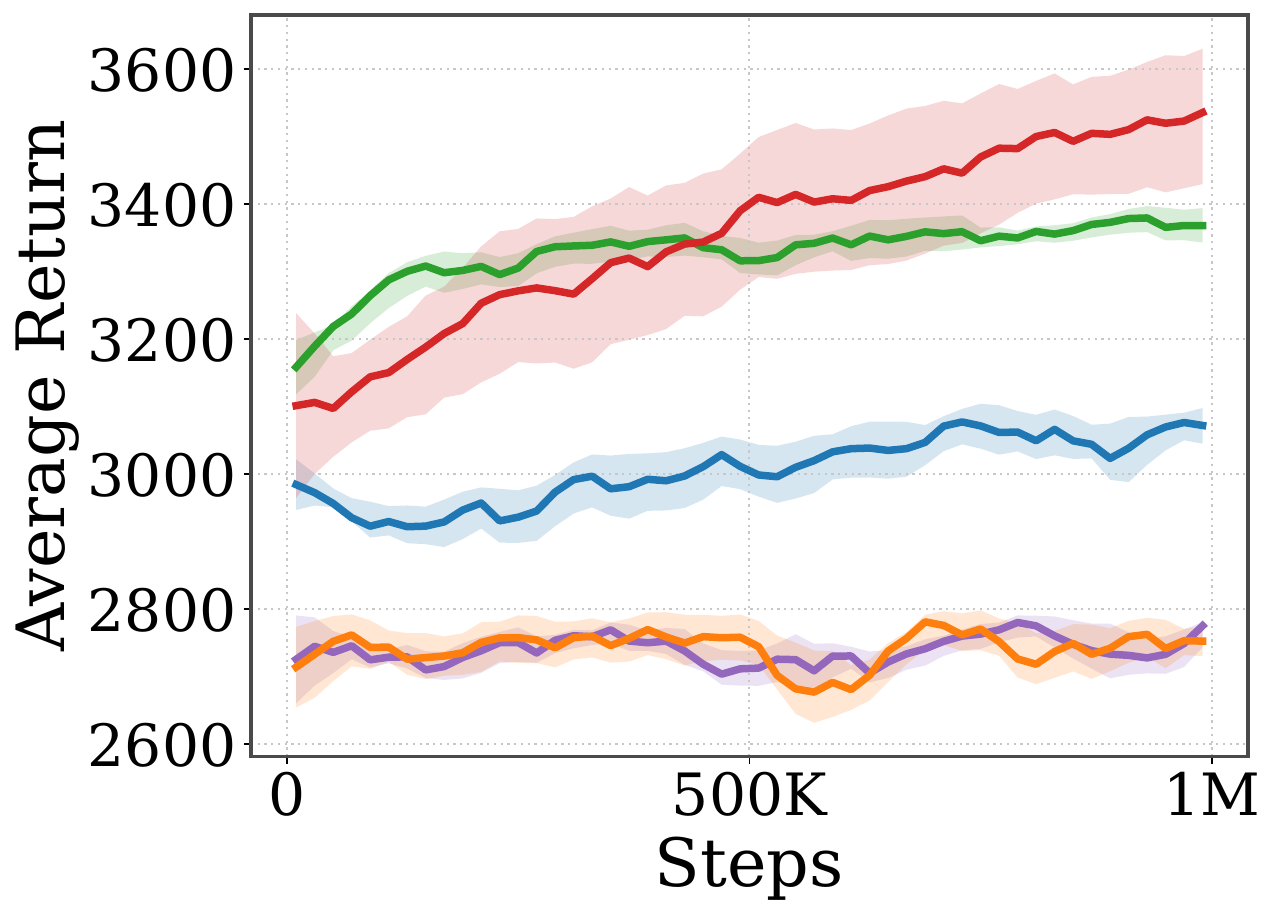}
    }
    \hfill
    \subfloat[HC 6x1 Med-Exp]{%
        \includegraphics[width=0.24\textwidth]{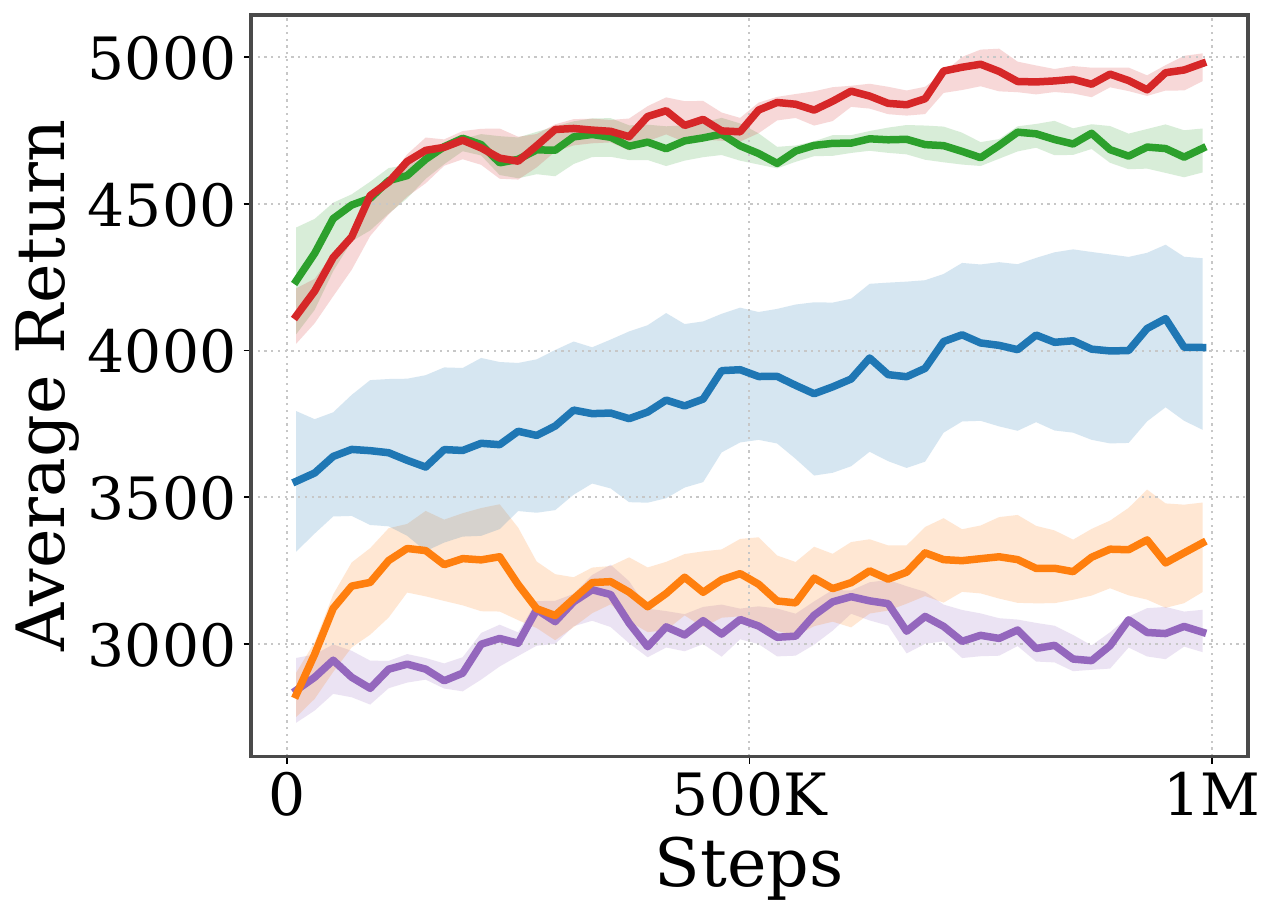}
    }

    \subfloat[HC 6x1 Med-Rep]{%
        \includegraphics[width=0.24\textwidth]{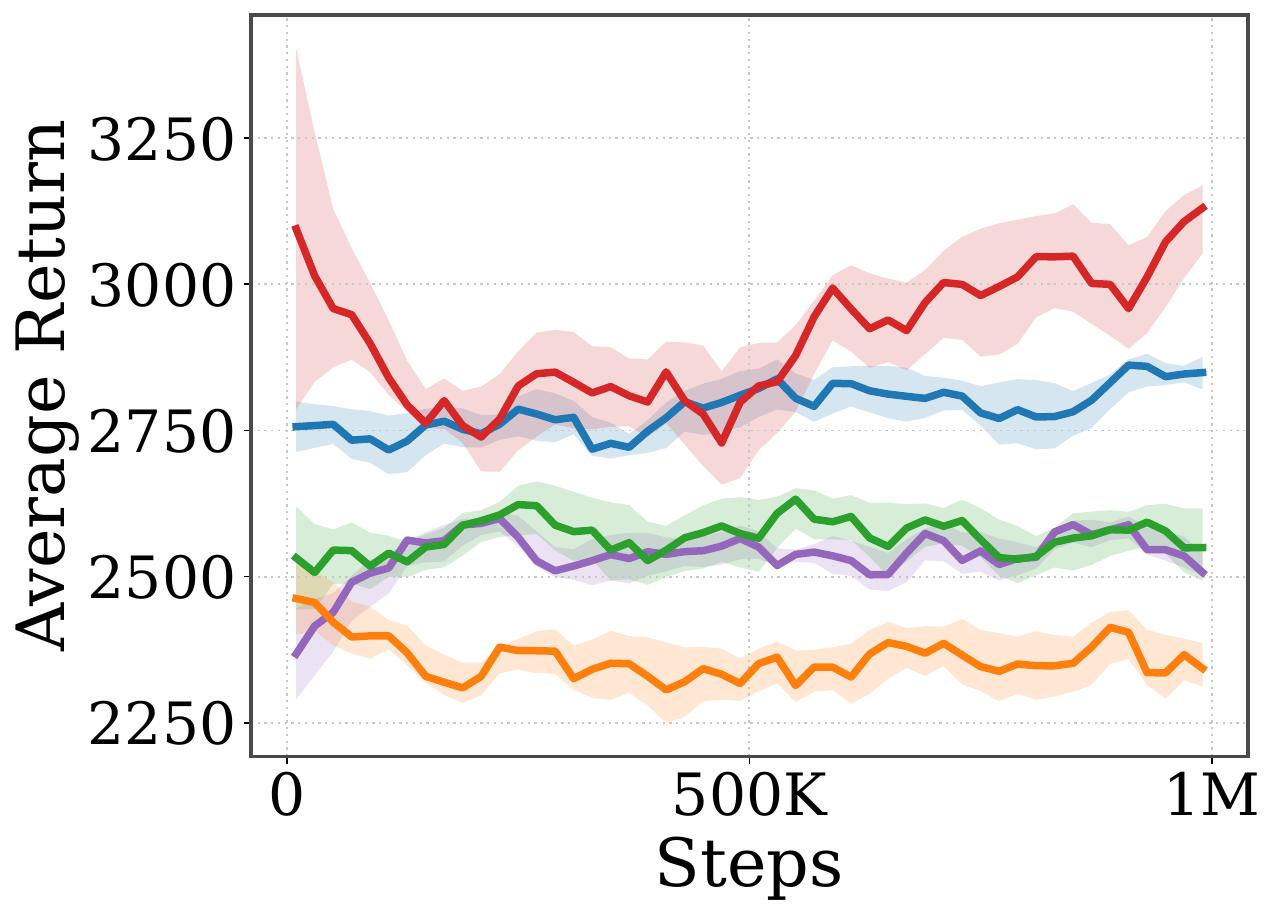}
    }
    \hfill
    \subfloat[Hopper 3x1 Medium]{%
        \includegraphics[width=0.24\textwidth]{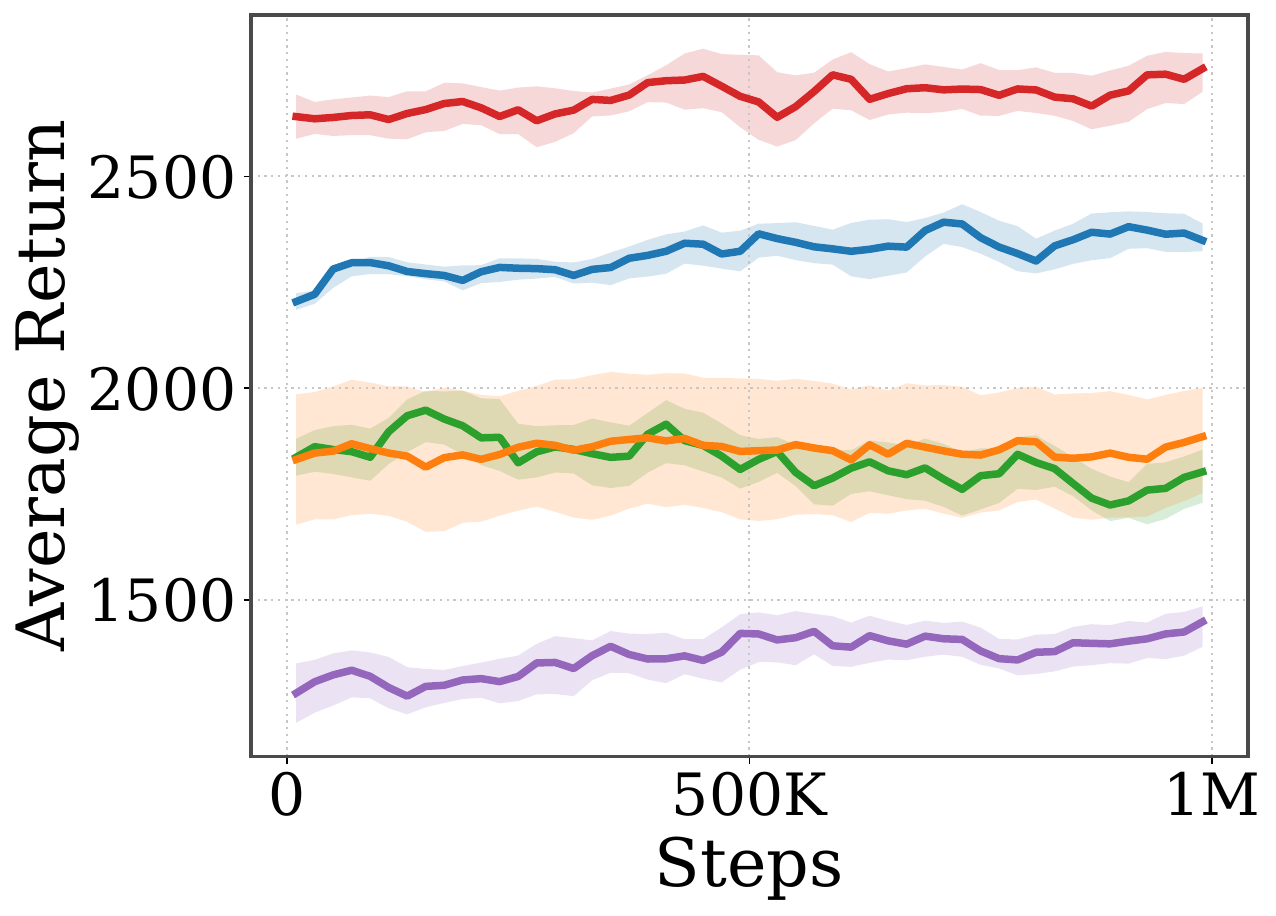}
    }
    \hfill
    \subfloat[Hopper 3x1 Med-Rep]{%
        \includegraphics[width=0.24\textwidth]{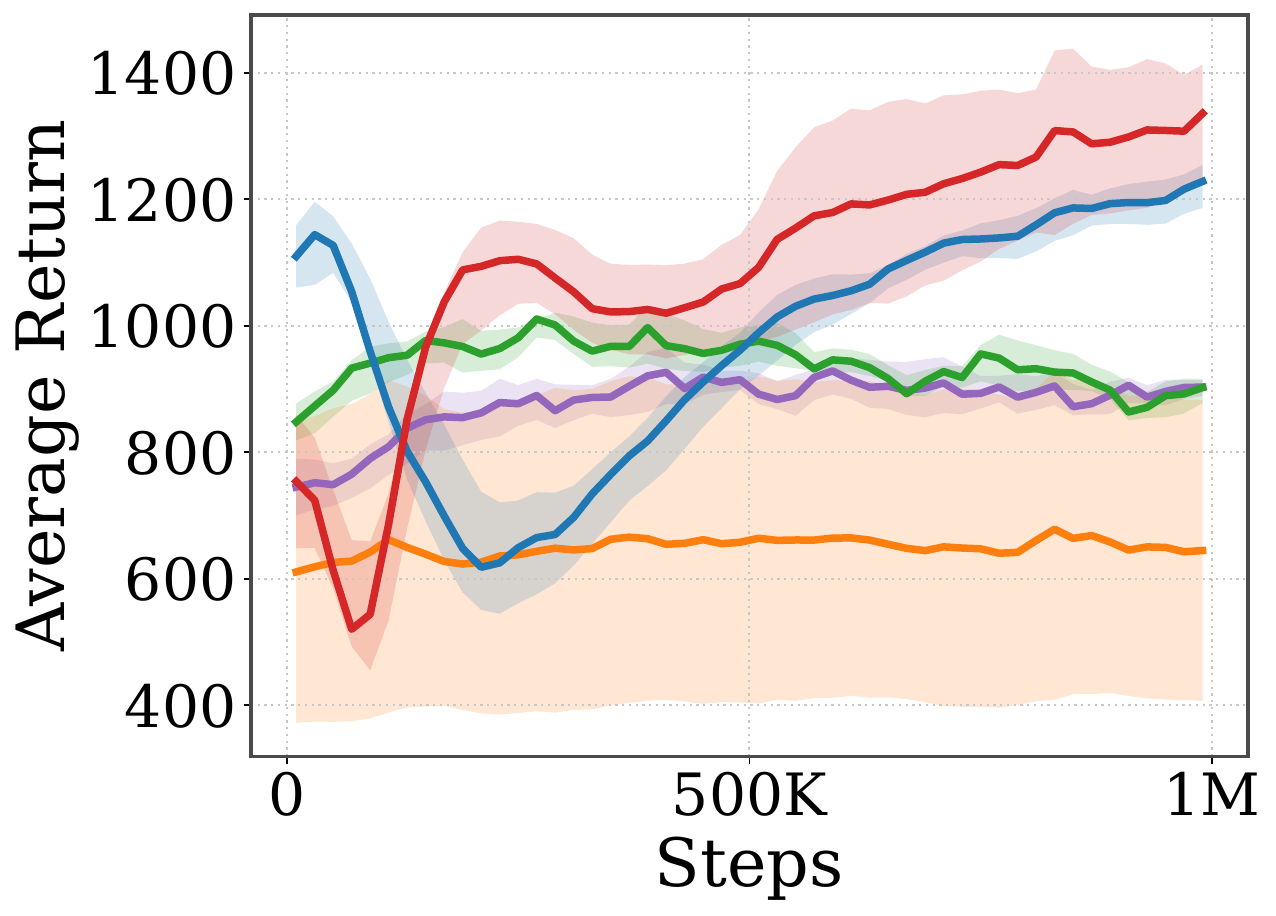}
    }
    \hfill
    \subfloat[Walker 6x1 Medium]{%
        \includegraphics[width=0.24\textwidth]{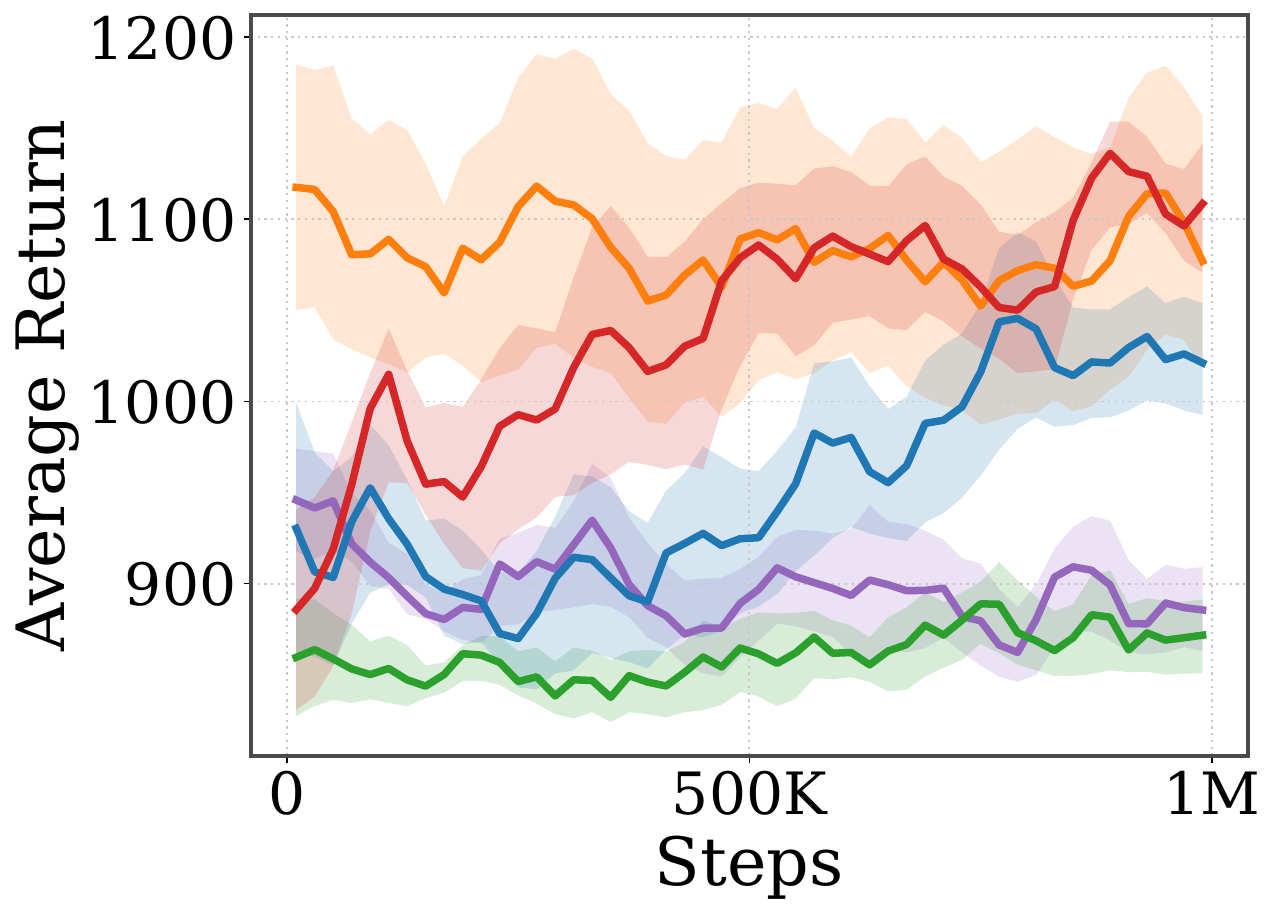}
    }

    \caption{Experimental results in MA-MuJoCo benchmarks, where HC, Med-Rep, and Med-Exp denote HalfCheetah, Medium-Replay, and Medium-Expert, respectively. The shaded region indicates the standard deviation across 5 seeds.}
    \label{fig:learning_curves}
\end{figure*}

\subsection{Probabilistic Pruning Mechanism}
\label{subsec:pruning}
We justify the mathematical design of the probabilistic pruning mechanism established in Eq.~\ref{eq:pruning} by casting it as an instance of constrained entropy-regularized value maximization. 
In contrast to rigid, deterministic top-$k$ selection, probabilistic pruning sustains crucial exploratory search diversity by sampling candidates from an optimization-derived distribution. 
This formulation effectively safeguards the active beam against premature modal collapse induced by localized value-estimation biases, ensuring that lower-scoring yet highly plausible joint-action profiles retain a non-zero survival probability. 
Formally, consider a single decision step at state $s$ over a finite candidate set $\mathcal{C} \subseteq \mathcal{A}_{\text{hyb}}$, and let $\Delta(\mathcal{C})$ denote the probability simplex over $\mathcal{C}$, where $P \in \Delta(\mathcal{C})$ represents the selection distribution.

\begin{definition}[Entropy-Regularized Objective]
\label{def:entropy_obj}
The optimal candidate selection policy is defined as the discrete distribution $P^*$ that maximizes the expected joint value function regularized by Shannon entropy:
\begin{equation}
    P^* = \operatorname*{argmax}_{P \in \Delta(\mathcal{C})} 
    \Big(
    \mathbb{E}_{\bm{a} \sim P} \left[ Q_{\text{tot}}(s, \bm{a}) \right] + \tau H(P)
    \Big),
\end{equation}
where $\tau > 0$ represents the temperature parameter, and $H(P) = - \sum_{\bm{a} \in \mathcal{C}} P(\bm{a} \mid s) \log P(\bm{a} \mid s)$ denotes the Shannon entropy of the selection profile.
\end{definition}

\begin{theorem}[Optimality of Boltzmann Distribution]
\label{thm:boltzmann_opt}
The unique closed-form solution to the constrained optimization problem articulated in Definition~\ref{def:entropy_obj} is identically the Boltzmann soft-max distribution:
\begin{equation}
    P^*(\bm{a}) = \frac{\exp\left(Q_{\text{tot}}(s, \bm{a})/\tau\right)}{\sum_{\bm{a}' \in \mathcal{C}} \exp\left(Q_{\text{tot}}(s, \bm{a}')/\tau\right)}.
\end{equation}
\end{theorem}
\begin{proof}
    Please refer to Appendix~\ref{app:proof_boltzmann} for the step-by-step derivation.
\end{proof}
Theorem~\ref{thm:boltzmann_opt} demonstrates that the probabilistic pruning strategy executing in Eq.~\ref{eq:pruning} precisely achieves the theoretical optimum for entropy-regularized search space filtration.

\begin{table*}[t]
\centering
\caption{Average episode returns in MA-MuJoCo benchmarks after 1 million online steps (mean $\pm$ std over 5 seeds).}
\label{tab:main_results}
\resizebox{\textwidth}{!}{%
\begin{tabular}{@{}l|c|l|ccccc@{}}
\toprule
\multicolumn{1}{c|}{Task} &
\begin{tabular}[c]{@{}c@{}} Config\end{tabular} &
\multicolumn{1}{c|}{Dataset} &
RLPD-MA & AWAC-MA & PEX-MA & PROTO-MA & Sim2O (Ours) \\
\midrule
\multirow{7}{*}{Ant} & \multirow{2}{*}{2x4} & medium
& $1407.25 \pm 3.45$ & $1420.05 \pm 2.00$ & $1733.31 \pm 54.66$ & $1419.49 \pm 2.17$ & $\mathbf{1798.05 \pm 73.32}$ \\
 &  & medium-replay
& $1358.21 \pm 11.32$ & $1499.56 \pm 7.06$ & $1759.10 \pm 29.32$ & $970.35 \pm 1.00$ & $\mathbf{1822.72 \pm 35.76}$ \\
\cline{2-8}
 & \multirow{3}{*}{4x2} & medium
& $1429.42 \pm 12.41$ & $1484.84 \pm 25.33$ & $1510.89 \pm 10.09$ & $1420.60 \pm 12.83$ & $\mathbf{1586.60 \pm 20.67}$ \\
 &  & medium-expert
& $1605.23 \pm 16.07$ & $1566.56 \pm 35.09$ & $1650.58 \pm 20.27$ & $1269.24 \pm 45.78$ & $\mathbf{1689.52 \pm 24.99}$ \\
 &  & medium-replay
& $1200.57 \pm 45.24$ & $1614.57 \pm 40.72$ & $1744.76 \pm 60.34$ & $966.37 \pm 38.71$ & $\mathbf{2024.66 \pm 93.21}$ \\
\cline{2-8}
 & \multirow{2}{*}{8x1} & medium-expert
& $1774.35 \pm 66.53$ & $1793.74 \pm 39.37$ & $1743.99 \pm 27.27$ & $1459.45 \pm 91.94$ & $\mathbf{1847.32 \pm 47.72}$ \\
 &  & medium-replay
& $1373.63 \pm 23.14$ & $1451.45 \pm 42.34$ & $1434.28 \pm 12.48$ & $1222.66 \pm 43.35$ & $\mathbf{1546.51 \pm 12.34}$ \\
\midrule
\multirow{4}{*}{HalfCheetah} & \multirow{4}{*}{6x1} & expert
& $3379.37 \pm 230.99$ & $4958.51 \pm 58.07$ & $4779.57 \pm 148.60$ & $3611.60 \pm 308.45$ & $\mathbf{5073.96 \pm 136.93}$ \\
 &  & medium
& $2792.85 \pm 27.53$ & $3367.92 \pm 61.37$ & $3069.09 \pm 68.02$ & $2752.22 \pm 51.73$ & $\mathbf{3543.84 \pm 227.67}$ \\
 &  & medium-expert
& $3026.93 \pm 179.88$ & $4709.65 \pm 159.62$ & $4010.15 \pm 673.03$ & $3367.27 \pm 317.99$ & $\mathbf{4994.26 \pm 88.66}$ \\
 &  & medium-replay
& $2488.39 \pm 32.74$ & $2550.09 \pm 150.14$ & $2850.57 \pm 82.68$ & $2327.58 \pm 85.85$ & $\mathbf{3145.98 \pm 109.50}$ \\
\midrule
\multirow{2}{*}{Hopper} & \multirow{2}{*}{3x1} & medium
& $1464.15 \pm 100.68$ & $1810.82 \pm 141.81$ & $2337.35 \pm 54.49$ & $1894.33 \pm 270.28$ & $\mathbf{2771.12 \pm 78.34}$ \\
 &  & medium-replay
& $903.51 \pm 23.74$ & $908.89 \pm 29.24$ & $1236.39 \pm 80.17$ & $645.91 \pm 531.63$ & $\mathbf{1354.49 \pm 196.50}$ \\
\midrule
Walker2d & 6x1 & medium
& $884.70 \pm 53.77$ & $872.62 \pm 46.00$ & $1017.93 \pm 67.91$ & $1062.81 \pm 154.36$ & $\mathbf{1116.75 \pm 94.15}$ \\
\bottomrule
\end{tabular}%
}
\end{table*}

\begin{figure}[t]
    \centering
    \vspace{2mm}
    \includegraphics[width=0.9\columnwidth]{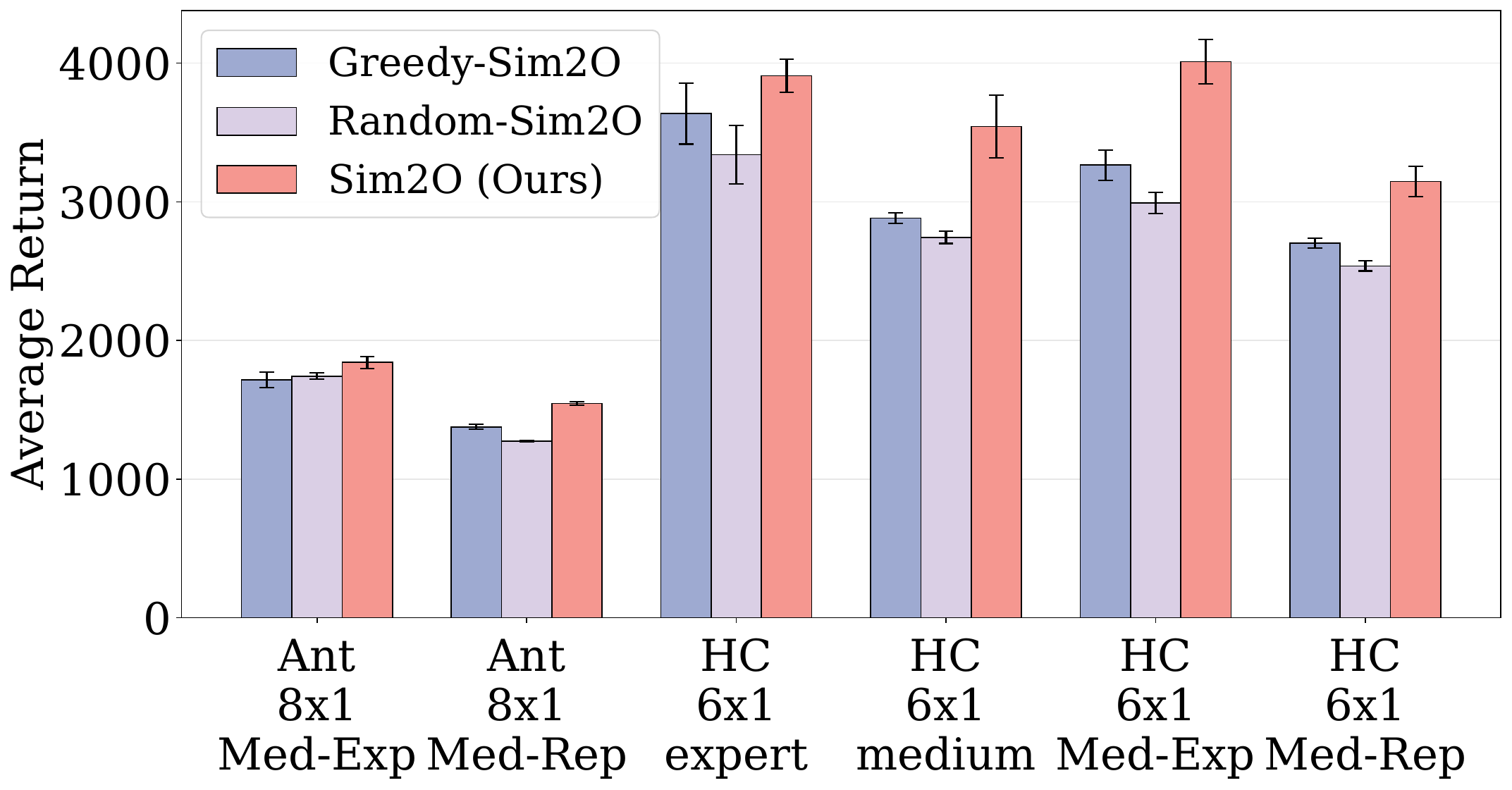}
    \caption{Experimental results of the module ablation for Sim2O. HC stands for HalfCheetah, Med-Rep stands for medium-replay, and Med-Exp stands for medium-expert.}
    \label{fig:module_ablation}
\end{figure}

\subsection{Suboptimality Analysis}
\label{subsec:bound}
We subsequently characterize the theoretical performance guarantees of Coordinated Beam Search by analyzing how closely its output approximates the global optimal solution within the fine-grained hybrid space $\mathcal{A}_{\text{hyb}}$. 
Specifically, we derive an analytical upper bound on the approximation gap, demonstrating that the suboptimality remains strictly controlled and asymptotically well-behaved. 
The total suboptimality error incurred by the joint action profile executed via CBS can be decomposed into two orthogonal sources: (i) \textit{search truncation error}, which measures whether the true global optimum survives the sequential pruning steps to remain in the final beam $\mathcal{B}_N$, and (ii) \textit{selection error}, which quantifies the expected value loss stemming from sampling from the randomized distribution over the beam rather than greedily exploiting the single best candidate. We formalize this structural error decomposition below.

Let $Q^* = \max_{\bm{a} \in \mathcal{A}_{\text{hyb}}} Q_{\text{tot}}(s, \bm{a})$ denote the global optimal value function. 
For the converged final beam set $\mathcal{B}_N$, let $P_{\text{sel}}$ represent the restricted Boltzmann distribution defined in Eq.~\ref{eq:pruning} conditioned exclusively on $\mathcal{B}_N$.
\begin{definition}[Error Decomposition]
\label{def:error_decomp}
The total suboptimality error $\varepsilon(k)$ is decomposed into the following additive terms:
\begin{equation}
\label{eq:error_decomp}
\varepsilon(k)
\;=\;
\varepsilon_{\text{search}}(k)
\;+\;
\varepsilon_{\text{prob}}(k),
\end{equation}
which are explicitly formulated as:
\begin{align}
\varepsilon_{\text{search}}(k)
&\;=\;
Q^* - \max_{\bm{a} \in \mathcal{B}_N} Q_{\text{tot}}(s, \bm{a}),
\\[8pt]
\varepsilon_{\text{prob}}(k)
&\;=\;
\max_{\bm{a} \in \mathcal{B}_N} Q_{\text{tot}}(s, \bm{a})
-
\mathbb{E}_{\bm{a} \sim P_{\text{sel}}}
\big[ Q_{\text{tot}}(s, \bm{a} \big].
\end{align}
Here, the search error $\varepsilon_{\text{search}}(k)$ captures the truncation gap between the global space and the preserved candidate subspace, while the selection error $\varepsilon_{\text{prob}}(k)$ isolates the expected performance penalty caused by the probabilistic sampling rule over $\mathcal{B}_N$.
\end{definition}

To establish the main bound, we assume that the global joint action-value function is coordinate-wise Lipschitz continuous. 
Specifically, let $L_Q \ge 0$ denote the Lipschitz constant such that for any state $s$ and any pair of joint action profiles $\bm{a}, \bm{a}' \in \mathcal{A}_{\text{hyb}}$ that differ by the action component of exactly one individual agent, we have:
\begin{equation}
\bigl| Q_{\text{tot}}(s, \bm{a}) - Q_{\text{tot}}(s, \bm{a}') \bigr| \le L_Q.
\end{equation}

\begin{theorem}[Suboptimality Bound]
\label{thm:bound}
The expected total suboptimality gap of Coordinated Beam Search satisfies the following upper bound:
\begin{equation}
    \mathbb{E}[\varepsilon(k)] \;\leq\; N^2 L_Q \left(1 - \frac{e^{-N L_Q/\tau}}{2k}\right)^k \;+\; \tau \log k.
    \label{eq:suboptimality_bound}
\end{equation}
\end{theorem}
\begin{proof}
    Please refer to Appendix~\ref{app:proof_gap} for the comprehensive mathematical derivation.
\end{proof}
Theorem~\ref{thm:bound} establishes a comprehensive, non-asymptotic guarantee for CBS. 
The first term, representing the search truncation boundary, monotonically contracts as the configured beam width $k$ scales up, indicating that expanding the beam capacity rapidly minimizes the risk of prematurely discarding high-value coordinated branches during the sequential agent-wise iterations. 
The second term, $\tau \log k$, arises directly from the entropy-regularized selection mechanism; it acts as an exploratory variance penalty that scales logarithmically with the beam size, elegantly highlighting the intrinsic exploration-exploitation trade-off modulated by the temperature parameter $\tau$.

\begin{figure}[t]
    \centering
    \vspace{2mm}
    \includegraphics[width=0.9\columnwidth]{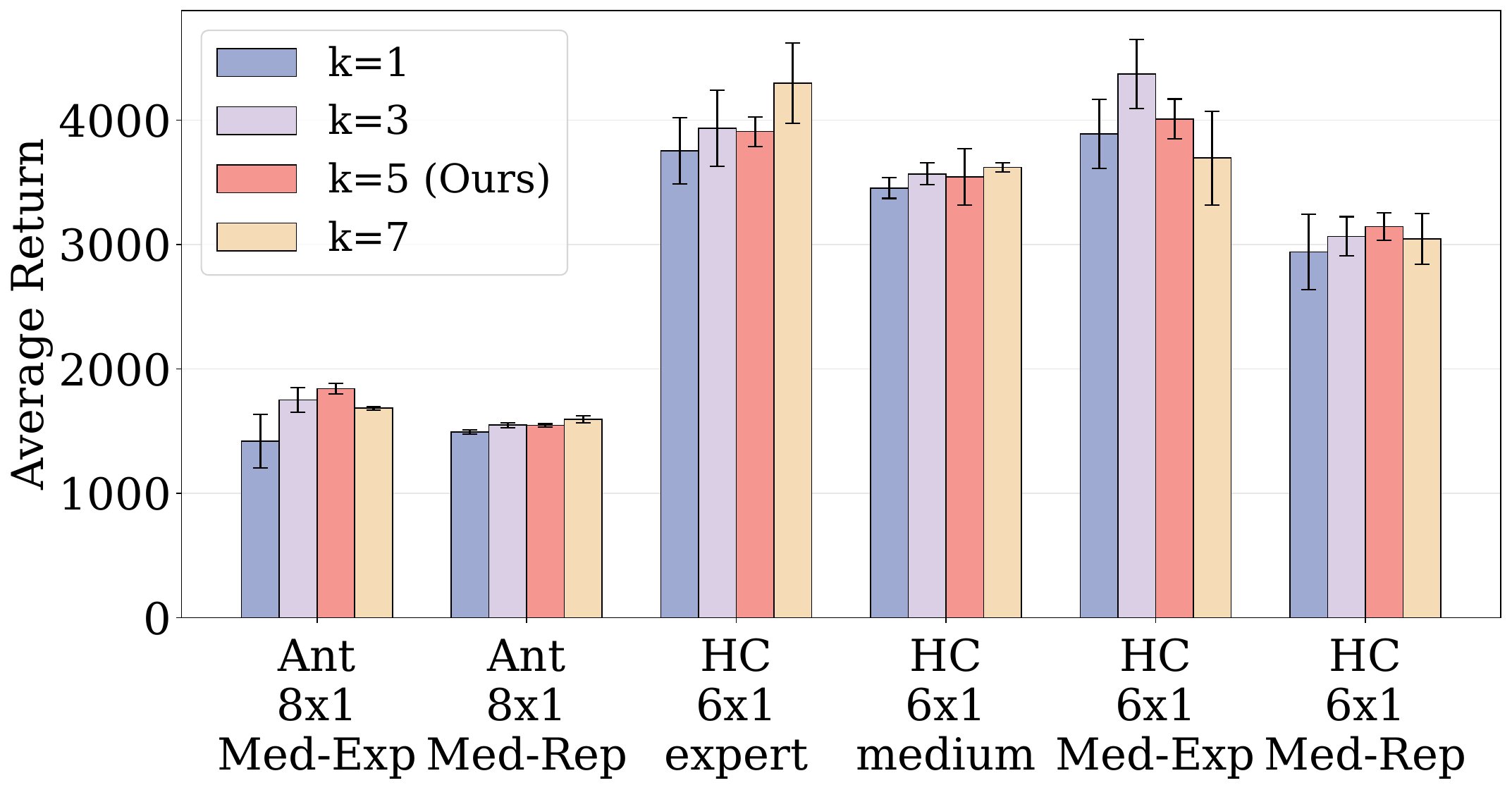}
    \caption{Experimental results of the beam width ablation for Sim2O by varying $k\in\{1,3,5,7\}$. HC stands for HalfCheetah, Med-Rep stands for medium-replay, and Med-Exp stands for medium-expert.}
    \label{fig:beam_width_ablation}
\end{figure}

\begin{table}[t]
\centering
\caption{Comparison with OMIGA-Finetune after 1 million online steps.}
\label{tab:ft_results}
\begin{tabular}{lcc}
\toprule
Environment & OMIGA-Finetune & Sim2O \\
\midrule
HC-6x1-medium
& $2499.59 \pm 47.34$
& $\textbf{3543.84} \pm \textbf{227.67}$ \\
HC-6x1-medium-replay
& $710.46 \pm 629.25$
& $\textbf{3145.98} \pm \textbf{109.50}$ \\
Ant-4x2-medium-replay
& $1187.17 \pm 60.34$
& $\textbf{2024.66} \pm \textbf{93.21}$ \\
Ant-8x1-medium-replay
& $1340.37 \pm 21.29$
& $\textbf{1546.51} \pm \textbf{12.34}$ \\
\bottomrule
\end{tabular}
\end{table}

\section{Experiments}
\label{sec:experiments}

In this section, we conduct an extensive empirical evaluation of Sim2O to address the following three core research questions: \textbf{Q1:} Does Sim2O achieve superior performance compared to state-of-the-art offline-to-online multi-agent baselines? \textbf{Q2:} What is the isolated algorithmic contribution of each individual component within Sim2O? \textbf{Q3:} Is the performance of Sim2O robust against diverse hyperparameter variations and environmental configurations?


\subsection{Experimental Setup}
\label{sec:setup}

\paragraph{Datasets.} 
As illustrated in Figure~\ref{fig:mamujoco_viz}, we evaluate our framework on the Multi-Agent MuJoCo benchmark \cite{dewitt2020decentralized}, specifically targeting high-dimensional continuous control tasks including \textit{Ant}, \textit{HalfCheetah}, \textit{Hopper}, and \textit{Walker2d}. To comprehensively evaluate the efficacy of offline-to-online adaptation, we construct offline trajectories from two distinct paradigms:
(1) \emph{Public Benchmarks}: We utilize the standard offline datasets curated by the OMIGA project \cite{wang2023globaltolocal}, which exhibit varying levels of data quality.
(2) \emph{Self-Collected Data}: To capture a wider spectrum of multi-agent coordination behaviors, we generate supplementary datasets utilizing HAPPO \cite{kuba2022happo}. To guarantee a rigorous evaluation across heterogeneous data distributions, our empirical framework spans four standard quality tiers: \textit{Medium}, \textit{Medium-Replay}, \textit{Medium-Expert}, and \textit{Expert}. Detailed dataset characteristics are provided in Appendix~\ref{app:datasets}.


\paragraph{Baselines.} 
To counteract the scarcity of dedicated offline-to-online MARL algorithms, we adapt four state-of-the-art single-agent offline-to-online counter-parts to the multi-agent paradigm. Specifically, we extend RLPD~\cite{ball2023rlpd}, AWAC~\cite{nair2020awac}, PEX~\cite{zhang2023pex}, and PROTO~\cite{li2023proto} to their multi-agent variants, denoted as \mbox{RLPD-MA}, \mbox{AWAC-MA}, \mbox{PEX-MA}, and \mbox{PROTO-MA}, respectively. To ensure a rigorous and equitable comparison, all baselines are executed under a decentralized training scheme equipped with shared replay buffers. Furthermore, during the offline pre-training phase, all methods employ OMIGA \cite{wang2023globaltolocal} to initialize policies on the static datasets, thereby guaranteeing a identical and high-quality behavioral initialization across all baselines. In the subsequent online adaptation phase, OMIGA serves as the underlying backbone for all implementations. We report the normalized episode return averaged across 5 independent random seeds, with exhaustive implementation details deferred to Appendix~\ref{app:implementation}.

\begin{table}[t]
\centering
\caption{Wall-clock statistics under different numbers of agents, where OMIGA-FT stands for OMIGA-Finetune}
\label{tab:wall_clock}
\begin{tabular}{llcc}
\toprule
Environment & Method & Action Time & Step Time \\
\midrule
\multirow{4}{*}{Ant-2x4-medium-replay}
& OMIGA-FT & $0.47 \pm 0.01$ & $7.42 \pm 0.22$ \\
& AWAC-MA     & $0.45 \pm 0.00$ & $6.30 \pm 0.05$ \\
& PEX-MA      & $1.50 \pm 0.01$ & $8.51 \pm 0.09$ \\
& Sim2O    & $2.09 \pm 0.02$ & $9.14 \pm 0.08$ \\
\midrule
\multirow{4}{*}{Ant-4x2-medium-replay}
& OMIGA-FT & $0.47 \pm 0.00$ & $12.19 \pm 0.13$ \\
& AWAC-MA     & $0.44 \pm 0.00$ & $5.95 \pm 0.02$ \\
& PEX-MA      & $1.42 \pm 0.01$ & $8.11 \pm 0.11$ \\
& Sim2O    & $3.45 \pm 0.02$ & $10.05 \pm 0.06$ \\
\midrule
\multirow{4}{*}{Ant-8x1-medium-replay}
& OMIGA-FT & $0.51 \pm 0.01$ & $13.19 \pm 0.39$ \\
& AWAC-MA     & $0.47 \pm 0.01$ & $6.46 \pm 0.18$ \\
& PEX-MA      & $1.51 \pm 0.01$ & $8.65 \pm 0.06$ \\
& Sim2O    & $6.74 \pm 0.11$ & $13.90 \pm 0.21$ \\
\bottomrule
\end{tabular}
\end{table}

\paragraph{Hyperparameters.}
\label{app:hyperparameters}
The core hyperparameter configurations for Sim2O and the baseline methods are formalized in Table~\ref{tab:hyperparams}. For the offline pre-training phase, we strictly adhere to the standardized protocols specified in the official OMIGA implementation. For the subsequent offline-to-online transition phase, we introduces three key algorithmic parameters: the beam width $k$, the pruning temperature $\tau$, and the offline-online replay mixing ratio $\rho$.

\begin{table}[t]
\centering
\caption{Average episode returns with various $\tau$.}
\label{tab:tau_sensitivity}
\begin{tabular}{lcc}
\toprule
Environment & $\tau$ & Episode return \\
\midrule
\multirow{3}{*}{HC-6x1-medium-replay}
& 1.0  & $3024.89 \pm 52.26$ \\
& 5.0  & $3495.16 \pm 7.28$ \\
& 10.0 & $3759.66 \pm 101.25$ \\
\midrule
\multirow{3}{*}{Ant-4x2-medium-replay}
& 1.0  & $1550.10 \pm 29.83$ \\
& 5.0  & $1900.82 \pm 138.67$ \\
& 10.0 & $1684.00 \pm 22.18$ \\
\bottomrule
\end{tabular}
\end{table}

\subsection{Main Results}
\label{sec:main_results}

\paragraph{Comparison with Baselines.}
We conduct comprehensive evaluations across diverse Multi-Agent MuJoCo tasks with heterogeneous dataset qualities and agent topologies. For instance, in the $\text{Ant-}4\times2$ task, the system is decomposed into 4 distinct agents, each controlling 2 joints. The empirical trajectories in Table~\ref{tab:main_results} and Figure~\ref{fig:learning_curves} demonstrate that Sim2O consistently delivers superior performance across all evaluated tasks, outperforming all baseline counterparts. We attribute this consistent superiority to our fine-grained agent-level filtering mechanism. This architectural design empowers individual agents to selectively preserve viable behavior primitives from the offline policy while aggressively exploring performance-enhancing trajectories, effectively decouples online adaptation from the baseline dataset quality limitations.

\paragraph{Comparison with OMIGA-Finetune.}
We further investigate whether a strong offline initialization is single-handedly sufficient for executing effective online adaptation. To this end, we contrast Sim2O against \mbox{OMIGA-Finetune}, which represents direct, unconstrained online fine-tuning from the pre-trained offline OMIGA checkpoint without any specialized adaptation constraints. As shown in Table~\ref{tab:ft_results}, direct continued fine-tuning from the offline checkpoint fails to yield reliable performance gains. This stark divergence suggests that the empirical benefits of Sim2O do not simply stem from a warm-started optimization baseline, but are uniquely driven by our proposed fine-grained hybrid action selection mechanism.

\paragraph{Wall-Clock Time Statistics.}
To evaluate the practical computational footprint of Sim2O, we report wall-clock runtime statistics. Here, \textit{Action Time} measures the average inference latency required for action selection per environment step, while \textit{Step Time} denotes the average total runtime per step, encapsulating action selection, environment interaction, and backpropagation updates. The empirical profiles are tabulated in Table~\ref{tab:wall_clock}. While the computational overhead of Sim2O scales approximately linearly with the agent count, the global runtime remains highly practical. Crucially, the growth rate of \textit{Step Time} is significantly smaller than that of \textit{Action Time}, demonstrating that the additional action-selection latency is sub-dominant and does not bottleneck the overall training loop.

\begin{table}[t]
\centering
\caption{Average episode returns on HalfCheetah-6x1-medium under different agent orderings.}
\label{tab:ordering_sensitivity}
\begin{tabular}{lc}
\toprule
Ordering & Episode return \\
\midrule
Forward & $3539.36 \pm 94.19$ \\
Reverse & $3623.12 \pm 30.28$ \\
Random & $3562.66 \pm 25.94$ \\
\bottomrule
\end{tabular}
\end{table}
\begin{table}[t]
\centering
\caption{Average episode returns on HalfCheetah-6x1-medium.}
\label{tab:beam_vs_exhaustive}
\begin{tabular}{lc}
\toprule
Method & Episode return \\
\midrule
Exhaustive selection & $3521.49 \pm 87.16$ \\
Sim2O & $3843.94 \pm 11.34$ \\
\bottomrule
\end{tabular}
\end{table}

\subsection{Ablation Study}
\label{sec:ablation study}

\paragraph{Module Ablation.}
We conduct systematic ablation studies to isolate the empirical contributions of critic-guidance and stochastic exploration outlined in Section~\ref{sec:cbs}. Specifically, we first replace the critic-guided beam search scoring function in Eq.~\ref{eq:pruning} with a random selection protocol, denoting this variant as \mbox{Random-Sim2O}. Second, we substitute the probabilistic pruning in Eq.~\ref{eq:pruning2} with a deterministic maximum-value selection, designated as \mbox{Greedy-Sim2O}. As illustrated in Figure~\ref{fig:module_ablation}, removing either architectural component triggers a consistent performance degradation across all tasks. These findings validate that both centralized critic coordination and stochastic exploration are imperative for robust offline-to-online adaptation in multi-agent environments.

\begin{table}[t]
\centering
\caption{Average episode returns under different offline pretraining steps.}
\label{tab:offline_pretraining_steps}
\begin{tabular}{lcc}
\toprule
Environment & Pretrain steps & Episode return \\
\midrule
\multirow{3}{*}{Ant-4x2-medium-replay}
& 1     & $1060.36 \pm 107.93$ \\
& 1000  & $1913.39 \pm 64.14$ \\
& 10000 & $2037.56 \pm 16.23$ \\
\midrule
\multirow{2}{*}{HalfCheetah-6x1-medium-replay}
& 1000   & $1192.26 \pm 1154.88$ \\
& 100000 & $3134.62 \pm 42.46$ \\
\bottomrule
\end{tabular}
\end{table}
\begin{table}[!t]
\centering
\caption{Average episode returns under different replay mixing ratios $\rho$.}
\label{tab:rho_sensitivity}
\begin{tabular}{lcc}
\toprule
Environment & $\rho$ & Episode return \\
\midrule
\multirow{3}{*}{Ant-4x2-medium-replay}
& 0.25 & $1735.82 \pm 194.43$ \\
& 0.5  & $1809.06 \pm 178.69$ \\
& 0.75 & $1517.59 \pm 55.55$ \\
\midrule
\multirow{3}{*}{HalfCheetah-6x1-medium-replay}
& 0.25 & $2394.11 \pm 887.12$ \\
& 0.5  & $1661.54 \pm 884.23$ \\
& 0.75 & $2269.06 \pm 1156.42$ \\
\bottomrule
\end{tabular}
\end{table}

\paragraph{Beam Width Ablation.}
A pivotal hyperparameter in Sim2O is the beam width $k$, which dictates the number of candidate hybrid joint actions retained during the coordinated beam search. To evaluate its sensitivity, we vary $k \in \{1, 3, 5, 7\}$. As shown in Figure~\ref{fig:beam_width_ablation}, increasing $k$ generally improves the final return by expanding the exploration frontier of hybrid joint compositions. However, these performance gains are non-monotonic; an excessively large beam width introduces optimization instability, which we hypothesize stems from value overestimation during the initial offline-to-online transition phase. Balancing this exploration-stability trade-off, we select $k=5$ as the unified default across all environments.

\paragraph{Temperature Sensitivity.}
The temperature parameter $\tau$ in Eq.~\ref{eq:pruning} modulates the stochasticity of the Boltzmann distribution used for probabilistic pruning and final action deployment. We evaluate $\tau \in \{1.0, 5.0, 10.0\}$ on the \mbox{HalfCheetah-6x1-medium-replay} and \mbox{Ant-4x2-medium-replay} tasks. The empirical outcomes summarized in Table~\ref{tab:tau_sensitivity} indicate that the optimal configuration of $\tau$ is environment-dependent, governing a critical balance between focusing on high-value candidates and preserving structural search diversity. Consequently, we deploy $\tau = 5.0$ as a unified default setting based on its robust performance across both tasks.

\paragraph{Search Strategy Ablation.}
To investigate whether the empirical edge of Sim2O merely arises from utilizing the centralized critic for action selection, we benchmark it against an exhaustive selection variant on \mbox{HalfCheetah-6x1-medium}. This variant enumerates the complete combinatorial hybrid joint action space and greedily executes the joint action that maximizes the estimated $Q_{\text{tot}}$. Both configurations are evaluated at 0.3 million online steps across 3 random seeds. Intriguingly, as shown in Table~\ref{tab:beam_vs_exhaustive}, Sim2O outperforms the exhaustive selection baseline. This result indicates that exhaustive optimization over the unconstrained hybrid action space severely exacerbates the critic's vulnerability to out-of-distribution (OOD) actions, thereby magnifying catastrophic value overestimation. In contrast, Coordinated Beam Search provides an implicit regularizing effect by progressively restricting the active candidate set, yielding substantially more reliable action selection during online adaptation.

\paragraph{Agent Ordering Ablation.}
Given that CBS expands hybrid candidates sequentially agent-by-agent, we verify whether the performance is sensitive to the chosen permutation order. We contrast forward, reverse, and dynamically randomized ordering on the \mbox{HalfCheetah-6x1-medium} task over 0.5 million online steps across 3 seeds. As demonstrated in Table~\ref{tab:ordering_sensitivity}, all three permutation strategies yield statistically indistinguishable performance. We therefore adopt random ordering by default to preclude any reliance on domain-specific structural orderings that may not generalize across environments.

\begin{table}[t]
    \centering
    \caption{Summary of hyperparameters employed in Sim2O. We adopt the standard OMIGA configuration for the offline phase and introduce specific parameters for the online fine-tuning stage.}
    \label{tab:hyperparams}
    \renewcommand{\arraystretch}{1.2}
    \begin{tabular}{lc}
        \toprule
        \textbf{Hyperparameter} & \textbf{Value} \\
        \midrule
        \multicolumn{2}{l}{\textit{Offline Phase}} \\ 
        \midrule
        Optimizer & Adam \\
        Learning Rate (Q-function) & $5 \times 10^{-4}$ \\
        Learning Rate (Actor) & $5 \times 10^{-4}$ \\
        Learning Rate (V-function) & $5 \times 10^{-4}$ \\
        Batch Size & 128 \\
        Discount Factor $\gamma$ & 0.99 \\
        Target Update Rate $\tau_{\text{soft}}$ & 0.005 \\
        Hidden Dimension & 256 \\
        Weight Network Hidden Dim. & 64 \\
        Regularization Parameter $\alpha$ (OMIGA) & 10  \\
        \midrule
        \multicolumn{2}{l}{\textit{Online Phase}} \\
        \midrule
        Beam Width $k$ & 5 \\ 
        Temperature $\tau$ & 5.0 \\
        Mixing Ratio $\rho$ & 0.5 \\
        Offline Policy Weights & Frozen \\
        Evaluation Episodes & 10 per checkpoint \\
        \bottomrule
    \end{tabular}
\end{table}

\paragraph{Offline Pre-training Ablation.}
The quality of the pre-trained critic directly dictates the reliability of value-guided search during the volatile early stages of online adaptation. To quantify this coupling, we vary the duration of offline pre-training steps prior to entering the online fine-tuning loop. Evaluations are conducted on \mbox{Ant-4x2-medium-replay} and \mbox{HalfCheetah-6x1-medium-replay}, with the final performance reported at 0.3 million online steps across 3 seeds. As evidenced by Table~\ref{tab:offline_pretraining_steps}, extended offline pre-training consistently correlates with superior online sample efficiency. This underscores that Sim2O critically benefits from an accurately regularized pre-trained $Q$-function, as the fidelity of the value-guided search hinges fundamentally on the precision of global value estimations.

\paragraph{Replay Mixing Ratio Ablation.}
Sim2O updates its network parameters using minibatches drawn from an interleaved mixture of offline and online experience, where the replay mixing ratio $\rho$ controls the relative gradient influence between the static dataset and active online exploration. We evaluate $\rho \in \{0.25, 0.5, 0.75\}$ on \mbox{Ant-4x2-medium-replay} and \mbox{HalfCheetah-6x1-medium-replay} for 0.5 million online steps over 3 seeds. As shown in Table~\ref{tab:rho_sensitivity}, Sim2O exhibits pronounced robustness across diverse mixing ratios, though the optimal trade-off point remains environment-specific. This demonstrates that $\rho$ acts as a regulatory knob balancing the retention of offline structural knowledge against the assimilation of novel online experiences. Accordingly, we maintain $\rho = 0.5$ as our standardized default.

\section{Conclusion}
In this paper, we investigated the largely unexplored challenge of offline-to-online adaptation in MARL. We proposed Sim2O, a principled yet straightforward framework that reformulates adaptation as a compositional process. By constructing hybrid joint actions from both offline and online proposals and evaluating them via a global value function, Sim2O efficiently identifies high-value team strategies without the need for complex auxiliary objectives. Our extensive experiments demonstrate that Sim2O significantly outperforms existing baselines, validating that a streamlined compositional approach is sufficient to bridge the gap between offline pre-training and online fine-tuning.


\bibliographystyle{IEEEtran}
\bibliography{references}

\appendix



\subsection{Proof of Theorem~\ref{thm:dominance}}
\label{app:proof_dominance}

\textbf{Theorem~\ref{thm:dominance}.}
\textit{For the true global joint action-value function $Q_{\text{tot}}(s,\bm{a})$, the optimization upper bound satisfies:
\begin{equation}
\max_{\bm{a}\in\mathcal{A}_{\text{hyb}}} Q_{\text{tot}}(s,\bm{a})
\;\ge\;
\max_{\bm{a}\in\mathcal{A}_{\text{sync}}} Q_{\text{tot}}(s,\bm{a}).
\end{equation}}

\begin{proof}
The proof proceeds by establishing a strict set-inclusion relation between the two joint action spaces under consideration. 
Recall from Definition~\ref{def:sync_space} that the synchronized joint action space is structured as $\mathcal{A}_{\text{sync}} = \{ \bm{a}_\beta, \bm{a}_\theta \}$, where $\bm{a}_\beta = \langle a^1_\beta, \dots, a^N_\beta \rangle$ and $\bm{a}_\theta = \langle a^1_\theta, \dots, a^N_\theta \rangle$ denote the homogeneous joint action vectors in which all individual agents uniformly execute actions derived from either the offline behavior policy or the active online policy, respectively. 
Conversely, the fine-grained hybrid joint action space $\mathcal{A}_{\text{hyb}} = \prod_{i=1}^{N} \mathcal{A}^i$ systematically encompasses all possible coordinate-wise combinations of individual alternative actions $a^i \in \{a^i_\beta, a^i_\theta\}$.

Consequently, by construction, both homogeneous profiles $\bm{a}_\beta$ and $\bm{a}_\theta$ constitute valid configurations within the complete Cartesian product space, yielding the fundamental subset inclusion relation:
\begin{equation}
\mathcal{A}_{\text{sync}} \subseteq \mathcal{A}_{\text{hyb}}.
\end{equation}
By the standard monotonic property of the supremum operator over nested sets, maximizing the global joint value function $Q_{\text{tot}}(s, \bm{a})$ over the superset $\mathcal{A}_{\text{hyb}}$ inherently yields a value no less than the localized optimization constrained within its subset $\mathcal{A}_{\text{sync}}$. The stated inequality therefore holds continuously, completing the proof.
\end{proof}

\subsection{Proof of Theorem~\ref{thm:boltzmann_opt}}
\label{app:proof_boltzmann}

\textbf{Theorem~\ref{thm:boltzmann_opt}.}
\textit{The unique closed-form solution to the constrained optimization problem articulated in Definition~\ref{def:entropy_obj} reduces exactly to the Boltzmann softmax distribution:
\begin{equation}
    P^*(\bm{a}) = \frac{\exp\left(Q_{\text{tot}}(s, \bm{a})/\tau\right)}{\sum_{\bm{a}' \in \mathcal{C}} \exp\left(Q_{\text{tot}}(s, \bm{a}')/\tau\right)}.
\end{equation}
}

\begin{proof}
We formalize the constrained optimization problem utilizing the method of Lagrange multipliers over the probability simplex. The corresponding Lagrangian function $\mathcal{L}(P, \lambda)$ is formulated as:
\begin{equation}
\mathcal{L}(P, \lambda) = \sum_{\bm{a} \in \mathcal{C}} P(\bm{a}) \big( Q_{\text{tot}}(s, \bm{a}) - \tau \log P(\bm{a}) \big) + \lambda \Big( \sum_{\bm{a} \in \mathcal{C}} P(\bm{a}) - 1 \Big),
\end{equation}
where $\lambda \in \mathbb{R}$ denotes the Lagrange multiplier associated with the normalization constraint.

Setting the first-order partial derivative of $\mathcal{L}(P, \lambda)$ with respect to each discrete probability coordinate $P(\bm{a})$ to zero yields the optimality condition:
\begin{equation}
\label{eq:fonc}
\frac{\partial \mathcal{L}}{\partial P(\bm{a})} = Q_{\text{tot}}(s, \bm{a}) - \tau \big( \log P(\bm{a}) + 1 \big) + \lambda = 0.
\end{equation}
Solving Eq.~\ref{eq:fonc} explicitly for $P(\bm{a})$ isolates the exponential functional form:
\begin{equation}
\label{eq:exponential_form}
P(\bm{a}) = \exp\left( \frac{\lambda - \tau}{\tau} \right) \exp\left( \frac{Q_{\text{tot}}(s, \bm{a})}{\tau} \right).
\end{equation}

To determine the scalar term wrapped by the multiplier $\lambda$, we invoke the fundamental simplex closure constraint $\sum_{\bm{a} \in \mathcal{C}} P(\bm{a}) = 1$. Summing both sides of Eq.~\ref{eq:exponential_form} over the entire candidate set $\mathcal{C}$ dictates:
\begin{equation}
\exp\left( \frac{\lambda - \tau}{\tau} \right) = \left( \sum_{\bm{a}' \in \mathcal{C}} \exp\left( \frac{Q_{\text{tot}}(s, \bm{a}')}{\tau} \right) \right)^{-1} \triangleq \frac{1}{Z(s)},
\end{equation}
where $Z(s)$ represents the state-dependent partition function. Substituting this normalizing factor back into Eq.~\ref{eq:exponential_form} directly evaluates to the optimal selection policy:
\begin{equation}
P^*(\bm{a}) = \frac{\exp\left( Q_{\text{tot}}(s, \bm{a}) / \tau \right)}{\sum_{\bm{a}' \in \mathcal{C}} \exp\left( Q_{\text{tot}}(s, \bm{a}') / \tau \right)}.
\end{equation}

Regarding uniqueness, we observe that the expected value term $\mathbb{E}_{\bm{a}\sim P}[Q_{\text{tot}}(s,\bm{a})]$ is strictly linear with respect to the density vector $P$, while the Shannon entropy $H(P) = -\sum_{\bm{a}} P(\bm{a})\log P(\bm{a})$ is strongly concave over its compact domain. Consequently, the combined objective function is strictly concave over the convex probability simplex $\Delta(\mathcal{C})$, guaranteeing that the stationary point $P^*$ is the unique global maximizer.
\end{proof}

\subsection{Proof of Theorem~\ref{thm:bound}}
\label{app:proof_gap}


\textbf{Theorem~\ref{thm:bound}.}
\textit{The expected total suboptimality gap of Coordinated Beam Search satisfies the following upper bound:
\begin{equation}
\mathbb{E}[\varepsilon(k)] \;\leq\; N^2 L_Q \left(1 - \frac{e^{-N L_Q/\tau}}{2k}\right)^k \;+\; \tau \log k.
\label{eq:suboptimality_bound}
\end{equation}}

\begin{proof}
We initiate the proof by recalling the structural error decomposition introduced in Eq.~\ref{eq:error_decomp}:
\begin{equation}
\varepsilon(k) = \varepsilon_{\text{search}}(k) + \varepsilon_{\text{prob}}(k).
\end{equation}
Due to the stochastic nature of the sequential sampling process in Coordinated Beam Search, the resulting optimality gap $\varepsilon(k)$ constitutes a random variable. Consequently, by invoking the linearity of expectation, we obtain:
\begin{equation}
\mathbb{E}[\varepsilon(k)] = \mathbb{E}[\varepsilon_{\text{search}}(k)] + \mathbb{E}[\varepsilon_{\text{prob}}(k)],
\end{equation}
where the constituent terms associated with the final candidate beam $\mathcal{B}_N$ are explicitly defined as follows:
\begin{align}
\varepsilon_{\text{search}}(k) &= Q^* - \max_{\bm{a} \in \mathcal{B}_N} Q_{\text{tot}}(s, \bm{a}), \\
\varepsilon_{\text{prob}}(k) &= \max_{\bm{a} \in \mathcal{B}_N} Q_{\text{tot}}(s, \bm{a}) - \mathbb{E}_{\bm{a} \sim P_{\text{sel}}}[Q_{\text{tot}}(s,\bm{a})].
\end{align}

\paragraph{Part 1: Bounding the Search Truncation Error $\mathbb{E}[\varepsilon_{\text{search}}(k)]$}
The search error $\varepsilon_{\text{search}}(k)$ is non-zero if and only if the global optimal joint action $\bm{a}^*$ is prematurely excluded from the final candidate set $\mathcal{B}_N$. We formally define this event as a \emph{search failure} and denote its probability by $P(\text{search failure})$.

Let $\langle z_0, \dots, z_N \rangle$ denote the sequence of decision ancestors corresponding to $\bm{a}^*$ within the combinatorial search tree. At each agent-wise step $j$, for $\bm{a}^*$ to be preserved throughout the optimization process, its prefix ancestor $z_j$ must be successfully selected into the beam $\mathcal{B}_j$ from the expanded candidate set $\mathcal{C}_j$. This selection probability is governed by the Boltzmann distribution:
\begin{equation}
P(z_j) = \frac{\exp\left(Q_{\text{tot}}(s, z_j)/\tau\right)}{\sum_{\bm{a}' \in \mathcal{C}_j} \exp\left(Q_{\text{tot}}(s, \bm{a}')/\tau\right)}.
\end{equation}
Let $Q_{\max, j} = \max_{\bm{a}' \in \mathcal{C}_j} Q_{\text{tot}}(s, \bm{a}')$ represent the local maximum within the current candidate set. It is worth noting that the candidate $\bm{a}'$ achieving this maximum and the true optimal ancestor $z_j$ may belong to distinct coordinates of the search tree and can differ by up to $N$ agents' actions. However, invoking the Lipschitz property of the value function, a single-agent deviation alters the joint Q-value by at most $L_Q$. 

Consequently, by the triangle inequality, the maximum value discrepancy between any two joint actions within the local space is bounded by $N L_Q$. Therefore, the energy gap in the Boltzmann exponent satisfies:
\begin{equation}
Q_{\max, j} - Q_{\text{tot}}(s, z_j) \le N L_Q.
\end{equation}
Substituting this upper bound into the density expression yields the following lower bound for tracking the optimal trajectory:
\begin{equation}
P(z_j) \ge \frac{1}{2k} \exp\left( - \frac{Q_{\max, j} - Q_{\text{tot}}(s, z_j)}{\tau} \right) \ge \frac{1}{2k} e^{-N L_Q/\tau}.
\end{equation}
The probability that $z_j$ is excluded after $k$ independent parallel draws (with replacement) is at most $(1 - P(z_j))^k \le (1 - \frac{1}{2k}e^{-N L_Q/\tau})^k$. Applying the union bound over the $N$ sequential decision steps, the total probability of a search failure is bounded by:
\begin{equation}
P(\text{search failure}) \le N \left( 1 - \frac{e^{-N L_Q/\tau}}{2k} \right)^k.
\end{equation}
In the event of a search failure, the selected joint action can diverge from the global optimum by the choices of at most $N$ agents. Since each single-agent deviation contributes at most $L_Q$ to the value difference, the resulting suboptimality gap is bounded by the global value span $N L_Q$. Taking the expectation over this boundary yields:
\begin{equation}
\begin{aligned}
\mathbb{E}[\varepsilon_{\text{search}}(k)]
&\le P(\text{search failure}) \cdot (N L_Q) \\
&\le N^2 L_Q \left( 1 - \frac{e^{-N L_Q/\tau}}{2k} \right)^k.
\end{aligned}
\label{eq:bound_part1_revised}
\end{equation}

\paragraph{Part 2: Bounding the Selection Variance Error $\mathbb{E}[\varepsilon_{\text{prob}}(k)]$}
We proceed to bound the discrepancy between the maximum value preserved within the beam and the expected value of the deployed action. We utilize a fundamental information-theoretic identity relating the expected value under the Boltzmann distribution $P_{\text{sel}}$ to the log-sum-exp function (i.e., the soft maximum):
\begin{equation}
\label{eq:app_identity_revised}
\mathbb{E}_{\bm{a} \sim P_{\text{sel}}}[Q_{\text{tot}}(s,\bm{a})] = \tau \log \left( \sum_{\bm{a}\in \mathcal{B}_N} \exp\left(\frac{Q_{\text{tot}}(s,\bm{a})}{\tau}\right) \right) - \tau H(P_{\text{sel}}),
\end{equation}
where $H(P_{\text{sel}})$ denotes the Shannon entropy of the sampling policy.

The log-sum-exp term is naturally lower-bounded by its maximal scalar element. Defining the localized greedy optimal value as $Q_B^\star = \max_{\bm{a} \in \mathcal{B}_N} Q_{\text{tot}}(s, \bm{a})$, we have:
\begin{equation}
\sum_{\bm{a}\in \mathcal{B}_N}\exp\left(\frac{Q_{\text{tot}}(s,\bm{a})}{\tau}\right) \ge \exp\left(\frac{Q_B^\star}{\tau}\right).
\end{equation}
Applying the monotonic logarithm transformation and scaling by the temperature $\tau$ yields:
\begin{equation}
\tau \log \left( \sum_{\bm{a}\in \mathcal{B}_N}\exp\left(\frac{Q_{\text{tot}}(s,\bm{a})}{\tau}\right) \right) \ge \tau \log \exp\left(\frac{Q_B^\star}{\tau}\right) = Q_B^\star.
\end{equation}
Substituting this inequality back into Eq.~\ref{eq:app_identity_revised} results in:
\begin{equation}
\begin{split}
&\mathbb{E}_{\bm{a} \sim P_{\text{sel}}}[Q_{\text{tot}}(s,\bm{a})] \ge Q_B^\star - \tau H(P_{\text{sel}}) \implies\\ &\underbrace{Q_B^\star - \mathbb{E}_{\bm{a} \sim P_{\text{sel}}}[Q_{\text{tot}}(s,\bm{a})]}_{\varepsilon_{\text{prob}}(k)} \le \tau H(P_{\text{sel}}).   
\end{split}
\end{equation}
This directly establishes the deterministic upper bound $\varepsilon_{\text{prob}}(k) \le \tau H(P_{\text{sel}})$.

Recall that for any discrete distribution $P_{\text{sel}}$ supported on a finite set $\mathcal{B}_N$, the Shannon entropy is strictly maximized by the uniform distribution, which is bounded specifically by $\log |\mathcal{B}_N|$. Given that the operational beam size is constrained by the width parameter $k$ (i.e., $|\mathcal{B}_N| \le k$), it follows that:
\begin{equation}
H(P_{\text{sel}}) \le \log |\mathcal{B}_N| \le \log k.
\end{equation}
Taking the expectation across this boundary yields the targeted selection error bound:
\begin{equation}
\mathbb{E}[\varepsilon_{\text{prob}}(k)] \le \tau \log k.
\label{eq:prob_bound2_revised}
\end{equation}

Finally, by combining the search error bound from Eq.~\ref{eq:bound_part1_revised} with the selection error bound in Eq.~\ref{eq:prob_bound2_revised}, we arrive at the finalized expected total suboptimality gap:
\begin{equation}
\mathbb{E}[\varepsilon(k)] \le N^2 L_Q \left(1 - \frac{e^{-N L_Q/\tau}}{2k}\right)^k + \tau \log k.
\end{equation}
The proof is now complete.
\end{proof}

\subsection{Multi-Agent MuJoCo Benchmark}
\label{app:mamujoco}

Multi-Agent MuJoCo~\cite{dewitt2020decentralized} is a standard continuous multi-agent control benchmark wherein a single robotic system is decomposed into multiple cooperative agents, each controlling a localized subset of joints. 
In this study, we evaluate our method on four representative environments with various agent topologies: 
\textit{Ant} ($4\times2$, $8\times1$), 
\textit{HalfCheetah} ($2\times3$, $6\times1$), 
\textit{Hopper} ($3\times1$), 
and \textit{Walker2d} ($6\times1$), 
where the notation $a\times b$ specifies a configuration of $a$ cooperative agents each controlling $b$ joints. Across all tasks, the observation range hyperparameter \texttt{agent\_obsk} is explicitly set to $1$, thereby restricting each individual agent's partial observation to its immediate structural neighborhood.

\subsection{Offline Dataset Construction}
\label{app:datasets}

\subsubsection{Dataset Sources}

The offline datasets utilized in this study are derived from two primary paradigms. First, we incorporate the standardized public offline datasets released by the OMIGA benchmark~\cite{wang2023globaltolocal}. Second, to validate our framework against broader behavioral distributions, we augment these benchmarks with supplementary datasets collected by interacting with the environment via the HAPPO algorithm~\cite{kuba2022happo}.

For the data collection pipeline, we employ the official implementation of HAPPO, adhering strictly to the default hyperparameter configurations without any environment-specific optimization. The specific dataset compositions are instantiated as follows:
\begin{itemize}
    \item \textbf{Expert}: Generated by extracting trajectories from a fully converged HAPPO policy.
    \item \textbf{Medium}: Obtained by early-stopping the training process, thereby capturing the behavioral characteristics of a partially trained policy.
    \item \textbf{Medium-Replay}: Comprises the entire history of transitions logged in the replay buffer throughout the intermediate training phase.
    \item \textbf{Medium-Expert}: Constructed by mixing equal proportions of trajectories from the \textit{Medium} and \textit{Expert} policies.
\end{itemize}

\subsubsection{Dataset Definitions}

We closely follow the standard D4RL taxonomy~\cite{fu2020d4rl} and adopt the following terminology across our empirical evaluation: \textit{Medium}, \textit{Medium-Replay}, \textit{Medium-Expert}, and \textit{Expert}. These designations correspond to trajectories sourced from partially optimized policies, historical training replay buffers, mixed-quality policy rollouts, and fully converged behavioral networks, respectively.

\subsection{Implementation Details}
\label{app:implementation}

\subsubsection{Sim2O Implementation}
\label{app:sim2o_impl}

We implement Sim2O directly within the OMIGA codebase, preserving the core structural backbone to ensure architectural consistency. The policy, local $Q$-value, and state-value networks are parameterized as three-layer Multi-Layer Perceptrons (MLPs) with 256 hidden units per layer and Rectified Linear Unit (ReLU) activations. For cooperative value decomposition, the mixing network is structured as a two-layer MLP with 64 hidden units per layer, also employing ReLU activations. All network parameters are optimized using the Adam optimizer.

During the offline-to-online adaptation phase, the pre-trained offline policy parameters are explicitly frozen. In contrast, the online policy and joint value networks are dynamically updated via minibatches sampled from a unified replay buffer that aggregates an interleaved mixture of offline and online data, adhering to the formal training protocol delineated in the main text.

\subsubsection{Baseline Implementations}
\label{app:baseline_impl}

We benchmark Sim2O against four prominent offline-to-online baseline algorithms: PEX~\cite{zhang2023pex}, AWAC~\cite{nair2020awac}, RLPD~\cite{ball2023rlpd}, and PROTO~\cite{li2023proto}. To establish a rigorous multi-agent comparison, we extend these algorithms to their multi-agent counterparts, denoted as \mbox{PEX-MA}, \mbox{AWAC-MA}, \mbox{RLPD-MA}, and \mbox{PROTO-MA}. For an equitable empirical foundation, all baselines are integrated into the unified OMIGA framework, sharing identical network capacities, optimization algorithms, replay buffer protocols, and training schedules unless explicitly stated otherwise.

The specific operational mechanics of each baseline extension are detailed below:
\begin{itemize}
    \item \textbf{\mbox{PEX-MA}}: Dynamically switches between the pre-trained offline policy and the evolving online policy at each environmental decision step, enforcing the selected policy choice uniformly across all cooperative agents.
    \item \textbf{\mbox{AWAC-MA}}: Replaces the standard policy update objective with the Advantage-Weighted Actor-Critic loss. This modification entails removing OMIGA-specific value regularizations and completely omitting the state-value network branch.
    \item \textbf{\mbox{RLPD-MA}}: Adheres to the classic offline-to-online fine-tuning pipeline, initializing both the policy and critic networks from the pre-trained offline checkpoint and subsequently updating them using a shared replay buffer with a balanced mix of offline and active online experiences.
    \item \textbf{\mbox{PROTO-MA}}: Augments the policy optimization objective with a prototype-based Kullback-Leibler (KL) regularization term derived from the offline data distribution, introducing an auxiliary reference policy that is tracked via an exponential moving average (EMA).
\end{itemize}

\FloatBarrier

\end{document}